\documentclass[lettersize,journal]{IEEEtran}

\usepackage{color}
\usepackage{times}
\usepackage{epsfig}
\usepackage{graphicx}
\usepackage{tabularx}
\usepackage{amsmath}
\usepackage{amssymb}
\usepackage{bbm}
\usepackage[misc]{ifsym}
\usepackage[export]{adjustbox}
\usepackage{rotating}
\usepackage{xcolor}
\usepackage[hyphens]{url}
\usepackage[ruled,linesnumbered]{algorithm2e}
\usepackage{graphicx} 
\usepackage{verbatim}
\usepackage{booktabs}
\usepackage{multirow}
\usepackage{makecell}
\usepackage{authblk}
\usepackage{hyphenat} 
\usepackage{pgf}
\usepackage{wrapfig}
\usepackage{amsthm}
\graphicspath{{figures/}}

\newlength\myindent
\usepackage{hyperref}
\hypersetup{
    colorlinks=true,
    linkcolor=blue,
    filecolor=magenta,      
    urlcolor=cyan,
    pdftitle={Overleaf Example},
    pdfpagemode=FullScreen,
    }

\newtheorem{theorem}{Theorem}
\newtheorem{lemma}{Lemma}

\makeatletter
\def\ps@IEEEtitlepagestyle{%
  \def\@oddhead{\mycopyrightnotice}%
  \def\@oddfoot{\hbox{}\@IEEEheaderstyle\leftmark\hfil\thepage}\relax
  \def\@evenhead{\@IEEEheaderstyle\thepage\hfil\leftmark\hbox{}}\relax
  \def\@evenfoot{}%
}
\def\mycopyrightnotice{%
  \begin{minipage}{\textwidth}
  \centering \scriptsize
  This article has been accepted for publication in the IEEE Internet of Things Journal. Copyright~\copyright~20XX IEEE.  Personal use of this material is permitted.  Permission from IEEE must be obtained for all other uses, in any current or future media, including reprinting/republishing this material for advertising or promotional purposes, creating new collective works, for resale or redistribution to servers or lists, or reuse of any copyrighted component of this work in other works.
  \end{minipage}
}
\makeatother

\begin{document}
\title{Instantaneous Wireless Robotic Node Localization Using Collaborative Direction of Arrival}
\author{Ehsan Latif and Ramviyas Parasuraman$^*$
\thanks{School of Computing, University of Georgia, Athens, GA 30602, USA.} 
\thanks{$^*$ Corresponding Author Email: ramviyas@uga.edu.}
}

\maketitle              

\begin{abstract}
Localizing mobile robotic nodes in indoor and GPS-denied environments is a complex problem, particularly in dynamic, unstructured scenarios where traditional cameras and LIDAR-based sensing and localization modalities may fail. Alternatively, wireless signal-based localization has been extensively studied in the literature yet primarily focuses on fingerprinting and feature-matching paradigms, requiring dedicated environment-specific offline data collection. We propose an online robot localization algorithm enabled by collaborative wireless sensor nodes to remedy these limitations. Our approach's core novelty lies in obtaining the Collaborative Direction of Arrival (CDOA) of wireless signals by exploiting the geometric features and collaboration between wireless nodes. The CDOA is combined with the Expectation Maximization (EM) and Particle Filter (PF) algorithms to calculate the Gaussian probability of the node's location with high efficiency and accuracy. The algorithm relies on RSSI-only data, making it ubiquitous to resource-constrained devices. We theoretically analyze the approach and extensively validate the proposed method's consistency, accuracy, and computational efficiency in simulations, real-world public datasets, as well as real robot demonstrations. The results validate the method's real-time computational capability and demonstrate considerably-high centimeter-level localization accuracy, outperforming relevant state-of-the-art localization approaches. 
\end{abstract}
%

%

\begin{IEEEkeywords}
Localization, Robots, Wireless Sensor Networks, Collaboration, Expectation Maximization, Particle Filter
\end{IEEEkeywords}


\section{Introduction}
\label{sec:intro}

A set of sensors, actuators, and mobile devices are connected to form an Internet of Things (IoT) system. Here, location information is critical for such operations, especially for wireless and mobile robotic nodes.
Node localization has been a challenging problem, especially in indoor environments. 
As such, Indoor localization has emerged as one of the most critical components in robotics, automation, and wireless systems. Here, one fundamental requirement is to provide an accurate and efficient localization system in a real-time (online) manner.
Furthermore, GPS-denied or dynamically changing environments pose additional challenges for mobile robot indoor localization \cite{survey}.

Sensors such as cameras, LIDAR, inertial measurement units (IMU), and their fusion have been exploited for obtaining accurate indoor localization of mobile devices \cite{canedo2016particle}. However, these technologies are expensive, non-applicable to resource-constrained devices and robots, and also suffer from various limitations, such as the requirement of proper lighting conditions in vision-based localization and structured non-dynamic surfaces for LIDAR-based perception.


\begin{figure}[t]
\centering
 \includegraphics[width=0.99\linewidth]{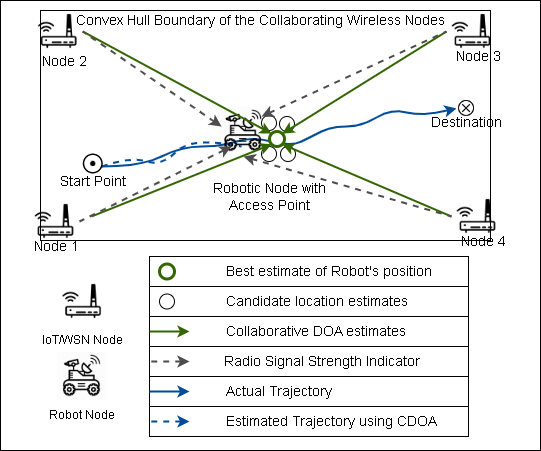}
 \caption{Overview of the proposed real-time CDOA-based localization of a robotic node following a trajectory with limited IoT/Wireless Nodes.}
 \label{fig:overview}
\end{figure}

On the other hand, wireless technologies such as Wi-Fi and Bluetooth are the most extensively utilized for indoor WLANs. The ubiquitous availability of Received Signal Strength Indicator (RSSI) measurements from Access Points (AP) or Wireless Sensor Nodes (WSNs) can be used for various objectives, including localization \cite{rizzo2021alternative,latif2022dgorl, RFID2021survey7}, multi-robot control \cite{luo2019multi,parasuraman2019consensus}, and communication optimization \cite{parasuraman2018kalman}. These advances provide opportunities to exploit the RSSI information from WSNs in aiding mobile robot localization.

Extant RSSI-based indoor positioning systems frameworks require an offline site survey to generate fingerprints and match the current real-time RSSI data to this database for positioning with a supervised machine learning algorithm (e.g., \cite{sadowski2020memoryless}). However, the fingerprinting approaches require a dedicated offline phase, in addition to the limitation of generalization, where they can be employed only for the specific environment where they are fingerprinted \cite{tao2018novel}. However, for mobile robot deployments, these limitations are not practical.

To address these gaps, we propose a novel algorithm for estimating the Collaborative Direction of Arrival (CDOA) estimated with the RSSI values obtained through collaboration among WSNs. 
The CDOA estimation is then integrated with two Bayesian framework variants for robust node localization: Expectation Maximization (EM) and Particle Filter (PF). 
See Fig.~\ref{fig:overview} for an overview of the proposed WSN collaboration-based wireless robotic node localization method.

The main contributions of this paper are outlined below.
\begin{enumerate}
   \item We propose a novel collaboration-aided mechanism for a mobile robot to collect RSSI data from the WSNs and estimate Wireless signal CDOA.
   \item We integrate the CDOA with a Bayesian framework for robot node localization. We propose two variants (Expectation Maximization and Particle Filter) to exploit the statistical EM method's accuracy and efficiency advantage the sampling-based PF method offers. 
   \item We theoretically analyze the properties of the proposed CDOA localization method in terms of localization consistency, accuracy, area coverage, scalability, and computational complexity.
    \item Through extensive experimental analysis in diverse setups enabled by numerical simulations, publicly available real-world datasets, and in-house robot hardware demonstrations, we evaluate the localization accuracy and efficiency of the proposed variants of CDOA-aided node localization. 
    \item We validate our approach by comparing with relevant non-fingerprinting methods from the recent literature such as the trilateration \cite{yang2020trilateration}, weighted centroid localization \cite{weightedCentroid2014survey}, differential RSSI \cite{podevijn2018comparison}, improved RSSI-based localization \cite{Xue2017improvedrssi}, smooth Particle Filter with Extended Kalman Filter \cite{zafari2018baysianfiltering}, and sparse Bayesian learning applied over the Direction of Arrival \cite{wang2019assistant} approaches.
    \item We open source all our codes and datasets (native Python implementations for the IoT and wireless sensor network community, as well as a ROS \cite{quigley2009ros} package for the robotics community) at \url{https://github.com/herolab-uga/cdoa-localization}. We believe this will enable the reproducibility and extension of our approach by the research community.
\end{enumerate}

The core novelty of our approach lies in that we employ a CDOA metric obtained through cooperative communication between the WSNs and the mobile robot instead of relying on the RSSI metric directly (as used in relevant methods in the literature). Further, we integrate and extensively evaluate the CDOA metric with Bayesian frameworks for robot node localization.
Our proposed methods achieve superior accuracy, efficiency, and robust localization performance through these novelties while enabling real-time efficiency compared to several state-of-the-art solutions.
A video demonstrating the CDOA approach implemented on real robots can be found at \url{https://www.youtube.com/watch?v=jVg2hzouO9E}.

\section{Related Work}
\label{sec:relatedwork}

According to a recent survey on Wi-Fi-based indoor positioning, \cite{liu2020survery4}, there are two categories of wireless localization solutions: \textit{Model-based and Survey-based}. Model-based approaches include trilateration using RSSI, triangulation using DOA, and Weighted Centroid using distance. Recent works include variants thereof, such as the filtered trilateration \cite{yang2020trilateration}, differential RSSI-based least squares estimation \cite{podevijn2018comparison}, and Expectation Maximization \cite{pajovic2015unsupervised}.
While the survey-based approaches provide high accuracy based on the precise fingerprints collected from the same environment through a dedicated offline process, they also come with high computation costs of prediction algorithms like the K-Nearest Neighbors \cite{subedi2020survey5}. Accordingly, we focus on model-based solutions.

In model-based approaches, multilateration and triangulation are the fundamental methods to predict the position of a wireless device (e.g., a mobile robot) using RSSI captured from multiple anchors/APs \cite{passafiume2016triangulation}. However, these methods suffer from co-linearity, ambiguous positioning, non-intersecting circles, etc. 
A recent survey \cite{zafari2019survey1} on model-based techniques confirms that the balance of accuracy and computing complexity is absent in the literature. Also, different variants of trilateration/triangulation can have different localization accuracy, resulting in inconsistency in its application. The weighted centroid method is less accurate for non-line-of-sight conditions, limiting its applicability. Further, some methods convert the raw RSSI measurements into distance estimates to use in multilateration algorithms, which suffer from the dependency on wireless channel parameters of the environment for RSSI to distance conversion.

\textcolor{black}{A novel trilateration algorithm is put forth by Yang et al. \cite{yang2020novel} for RSSI-based indoor localization of a target using the geometric relationships between transmitters and receivers. The precision of localization can be considerably impacted by environmental changes, such as obstructions and multipath propagation, because trilateration is sensitive to these. This approach also requires considerable calibration and exact distance estimation, which might be difficult in dynamic situations. A Gaussian Filtered RSSI-based Indoor Localization method utilizing Bootstrap filtering is presented by Wang et al. \cite{wang2021gaussian}. This technique uses particle filtering to determine the location of a target within a WLAN. Although it is more resistant to environmental noise, this method still relies on raw RSSI measurements, which are prone to interference, multipath propagation, and signal attenuation.}

\textcolor{black}{Pinto et al. \cite{pinto2021robust} use K-means clustering and Bayesian estimation to create a reliable RSSI-based indoor positioning system. This method seeks to increase the accuracy of localization by integrating unsupervised learning with probabilistic estimates. It is nonetheless susceptible to the drawbacks of RSSI measurements, such as susceptibility to environmental changes and the requirement for significant calibration. Bayesian filtering method also presented by Mackey et al.\cite{mackey2020improving} for enhancing BLE beacon proximity estimation accuracy. This technique improves the performance of BLE beacons, but is still vulnerable to interference from other wireless devices and may have decreased accuracy in NLOS scenarios and dynamic surroundings. Another method for estimating the path-loss exponent using Bayesian filtering is presented by Wojcicki et al. \cite{wojcicki2021estimation}. This method is intended to describe the propagation of signals in various contexts, increasing the precision of localization. The quality of the RSSI measurements, which the surroundings and signal attenuation can impact, is crucial to the path-loss exponent estimation's accuracy.}

\textcolor{black}{Combining data from many sources and using probabilistic techniques, the proposed CDOA methods minimize the drawbacks of the aforementioned methods. While CDOA-EM uses locations as samples to fill a grid and determine robot location using Gaussian probability, CDOA-PF increases robustness in complicated situations by repeatedly updating particles representing potential positions. Both strategies overcome the shortcomings of conventional RSSI-based techniques and offer more precise and dependable localization in the presence of NLOS, multipath propagation, and environmental changes.}

As an alternative to the RSSI metric, researchers have proposed the use of the Channel State Information (CSI) metric for robot localization systems \cite{jadhav2021toolbox,song2017csi,wang2016csi}. However, most CSI-based techniques involve extensive offline fingerprinting processes to improve accuracy. Moreover, as with RSSI-based metrics, the CSI metric is limited to very few radios and cannot be exploited for ubiquitous applications. DOA (or Angle of Arrival) based methods achieve higher localization accuracy than RSSI-based solutions. For instance, in research \cite{arbula2020sensors}, the authors used a sensor node equipped with Infrared light arrays to estimate the DOA of a mobile robot, which was used to achieve indoor localization with meter-level accuracy. Cooperatively localizing target nodes using multiple reference nodes with known locations has been explored. For instance, the authors in \cite{hassani2015cooperative} provided a distributed method for cooperatively estimating the DOA of an acoustic sensor network. In contrast, the authors in \cite{xu2015cooperative} used cooperative DOA from Ultra Wideband (UWB) radios to locate target nodes using many reference nodes. 

{\color{black}
UWB-based indoor localization systems, while highly promising, still face several limitations. The presence of multipath propagation and non-line-of-sight (NLOS) conditions can significantly affect the positioning accuracy \cite{alarifi2016ultra}. Yang et al. \cite{yang2022robust} proposed a UWB-based indoor localization with fewer nodes and utilized deep neural networking to avoid the effect of non-line-of-site; however, this solution requires offline data training and sampling overhead, which makes the system restricted to the trained environment. Additionally, deploying UWB anchors can be challenging in real-world environments due to their need for precise installation \cite{ridolfi2021uwb}. Further, the power consumption of UWB devices and their susceptibility to interference from other wireless systems can negatively impact their performance and scalability \cite{zafari2019survey1}. Integration of UWB-based systems with other sensing modalities can be difficult, as the fusion of data from different sources may be subject to noise and uncertainty \cite{wang2021high}.

A wireless sensor network with a few WSNs can overcome the limitations of UWB-based indoor localization  by leveraging cooperative sensor modalities to provide robust and accurate localization in complex environments. By fusing data from diverse sensing sources, WSNs can mitigate the effects of multipath propagation, NLOS conditions, and interference, enhancing the overall system performance \cite{celaya2020radio}.
Furthermore, commercial UWB-based localization solutions provide accuracy of up to 10cm and connection stability at the cost of high computation power and expensive anchor node solution, which makes it impractical for swarm robots ~\cite{starks2023heroswarm} with limited computational power and can only possess wireless connectivity. The proposed solution provides highly scalable, computationally efficient, and  high localization accuracy for small to large-scale multi-robot systems.
}

Estimating a mobile node's position using only a few reference nodes with high accuracy and efficiency is achievable in wireless sensor networks. Wang et al. \cite{wang2019assistant} proposed sparse Bayesian learning for robust DOA estimation with only a few base station nodes. But, their implementation assumed multiple antennas at each base station, realizing an EM-based DOA estimation and eventual vehicle localization using the DOA triangulation.
\textcolor{black}{ Wang's proposed solution is computationally expensive and requires high-end base stations, which makes it impractical for small robots operating indoor environment.} On the contrary, in our work, we assume typical Wi-Fi sensors without having access to multi-antenna data, allowing ubiquitous integration with existing wireless sensor networks/IoT systems and computationally efficient online CDOA-based indoor localization.


Therefore, we propose a CDOA estimation using IoT or wireless nodes and fuse it with Bayesian approaches for high-accuracy localization of mobile robotic nodes. 
\textit{We depart from the literature in two different ways: 1) we use a collaborative mechanism between the WSNs to obtain the CDOA of wireless signals; 2) we estimate Gaussian probability on the CDOA estimates, adopting EM and PF Bayesian frameworks}. The localization system can be applied independently of the robot's motion model or combined with the robot's odometry, if available, to improve accuracy. Moreover, our approach uses only a few reference nodes and works on resource-constrained robotic nodes in real time. 
Our proposed approach is advantageous by reducing computational complexity without embedding external hardware and using bearing-only information (aided by the cooperative RSSI measurements). It achieves high accuracy even in the presence of signal noise. This way, our method balances the efficiency and accuracy of quick online operation without fingerprinting dependence.
While localization of static Access Points has been demonstrated using DOA \cite{parashar2020particle}, this is the first work that uses CDOA for robotic node localization demonstrated in real-world implementations.


\section{Problem Statement}
\label{sec:problem}
We look at the problem of a robot node localizing itself against its surroundings.
Here, a limited (smaller) number of WSNs or IoT nodes are distributed in the environment, and the mobile robot is mounted with an AP, which nearby WSNs can sense.
\textcolor{black}{The robot can operate within the sensing range of WSNs, which is assumed to be 40m; the robot is not restricted to be in the boundary of WSNs but confined to be within the range (e.g., it could be outside the convex of boundary).}

WSNs are assumed to be static, and their exact position is known to the robot for gradient calculation in the global frame. Furthermore, the robot is restricted to moving within the sensing range of connected WSNs.
The WSNs measure the RSSI values coming from the AP and communicate this information to the robot cooperatively (assuming the measurements and shared data are reasonably time-synchronized using NTP-like protocols). 
\textcolor{black}{Robot will use RSSI values to calculate the gradient and convert it into the CDOA with respect to the position of WSNs.}
The robot \(R\) keeps track of the trajectory along with the CDOA measurements as the tuple: \( m_l = \{ x_l , y_l , CDOA_l\}\), where \((x_l , y_l)\) is the location of the robot at location \(l\). 

The \textcolor{black}{objective} is to find the best estimate of the robot's location \textcolor{black}{\((x_{l}^*, y_{l}^* )\)}, which maximizes the probability of observing the measurement tuples when the robot is at the estimated location \(P(x_{l}, y_{l} \mid m_l, m_{l-1}, . . . , m_{l-M})\), where \textcolor{black}{$m_l$ is the sample for position $l$} and \(M\) is the number of previous samples considered along the completed trajectory so far, given that we employ an arbitrary method to estimate the CDOA. \textcolor{black}{Table~\ref{tab:notations} lists the key symbols and notations used in the paper.}

\begin{table}[t]
\begin{center}
\caption{Notations and their descriptions}
\label{tab:notations}
\vspace{-4mm}
\resizebox{\linewidth}{!}{
\color{black}
\begin{tabular}{|l|l|}
\hline
\textbf{Notation} & \textbf{Description} \\ \hline
$S_i$ & Measured RSSI at a wireless node $N_i$ \\ \hline
$\vec{g} = [g_x,g_y]$ & Gradient of the RSSI's signal strength \\ \hline
$(x_c,y_c)$ & Centroid position of a rectangular workspace \\ \hline
$\Delta_{X}$, $\Delta_{Y}$ & Distance between the WSNs along the x and y axes \\ \hline
$CDOA_l$ & CDOA of the signal at a position along the robot's path $l$ \\ \hline
$P_r$, $R_r$ & Initial and re-sampled list of particles in PF \\ \hline
$\omega$ & Weights associated with each sample \\ \hline 
$x_t$ & State at time $t$ \\ \hline
$E_r$ & List of grid positions in the EM algorithm \\ \hline
$w_{E_r}(x_t)$ & Gaussian Probability for each grid position in $E_r$ at time $t$ \\ \hline
$M$ & Number of previous samples considered \\ \hline
$\sigma$ & Standard deviation of error for all previous samples \\ \hline
$err_l^{k}$ & CDOA error for sample $k$ \\ \hline
$w_i$ & Weight of each particle \\ \hline
$w_i^*(q_i)$ & Normalized weight of each particle \\ \hline
\end{tabular}
}
\end{center}
\end{table}

\color{black}

\section{Proposed CDOA Approach}
\label{sec:approach}

Cooperative localization can be accomplished with a network of wireless nodes, where each node can sense the signal strength of the other node in the network. 
Our approach consists of two units: 1) we propose a CDOA estimation scheme from RSSI measurements with an assumption of the geometric model of the AP/WSN distribution in the environment, which is typical in the literature; 2) our solution deploys EM and PF-based localization of a mobile robot node using the CDOA.

In principle, we need at least three WSNs that form at least two noncollinear segments between them to measure a valid gradient inside the boundary created by the WSNs (\cite{twigg2012rss}). Having a higher number of WSNs will increase the robustness of the solution. 
It is possible to extend this setup where a mesh network is available with several known and unknown wireless nodes, but we limit our scope and experiments in this paper with four WSNs deployed at the corners of a robotic node workspace boundary (see Fig.~\ref{fig:overview}).

\textcolor{black}{The proposed method has a basis in WSNs collaboration to measure RSSI collaboratively and calculate CDOA. Alg.~\ref{alg:cdoa} provides an algorithmic pseudo-code of the CDOA estimation of the mobile robot using the surrounding WSNs.}
The first part of the Alg.~\ref{alg:cdoa} lays out the wireless network collaboration for determining CDOA at the robot.
All these computations are performed by the moving \textcolor{black}{AP-mounted robot}, which runs a centralized service and receives the RSSI information from all connected nodes that sense wireless signals independently in a synchronized manner.

\textcolor{black}{For collaborative measurement of RSSI on WSNs connected to an access point installed on a robot, time synchronization is essential. These nodes can synchronize their clocks within the \textit{time window} using the Network Time Protocol (NTP), ensuring precise and consistent RSSI readings throughout the whole network \cite{ranganathan2010time}. Using the time window idea, sensor nodes can send and receive RSSI data at predetermined intervals, effectively controlling communication and minimizing the possibility of collisions while enabling the system to coherently process and evaluate the obtained data \cite{fang2021trust}.}

For the algorithm to work, we need at least three spatially-distributed WSNs in the network, and the robot needs to be within the polygonal boundary of the WSNs.

\begin{algorithm}[t]
At every wireless sensor node\;
\Indp
$S_i \leftarrow []$, list for all RSSI values for node $i$\;
sample RSSI $\rightarrow S_i$\;
Initialize the time window to $T$\;
   \While{time window is \textit{open}}{
        $S_i \rightarrow$ average all RSSI received\;  
    }
Publish $S_i$ to the ROS network\;
\Indm
At the robot node\;
\Indp
 \For{every wireless node $i \in N={1,2,..,4}$ in WSNs}{
        Receive RSSI ($S_i$) locally\;
    }
Calculate RSSI gradients in Eq.~\eqref{eqn:gradient} using $S_i \in N$ \;
Calculate CDOA at the robot using Eq.~\eqref{eqn:doa}\;
\Indm
  \caption{CDOA: WSN Collaboration-aided DOA Estimation at the Mobile Robot}
  \label{alg:cdoa}
\end{algorithm}

RSSI can be modeled as a vector with two components, and the gradient concerning the center of the robot can be represented as \(\vec{g} =[g_x,g_y]\). 
One of the primary advantages of the central finite difference method is that it provides gradient estimation based on the received signal strength from geometrically oriented wireless nodes. After an appropriate gradient estimation, a receiver node (moving robot) can estimate the direction of arrival of signals based on the reference position for position estimation using EM and PF contrivance. 

In our current implementation, the CDOA of the mobile robot within the network is obtained from the geometric rule described in the central finite difference method \cite{parasuraman2013spatial}.
For a rectangular configured networked infrastructure with centroid position $(x_c,y_c)$, \textcolor{black}{refer to Fig.~\ref{fig:overview} where the RSSI value of Node 1 is $S_1$ and so on,} the RSSI gradient is calculated as:
\begin{eqnarray}
\begin{aligned}
    g_x =  \frac{S_3 - S_2}{2\Delta_{X}} + \frac{S_4 - S_1}{2\Delta_{X}} \; ; \; 
    g_y = \frac{S_2 - S_1}{2\Delta_{Y}} + \frac{S_3 - S_4}{2\Delta_{Y}}
    \label{eqn:gradient}
\end{aligned}
\end{eqnarray}
Here, \(\Delta_{X}\) is the distance between the wireless sensor's antennas along the x-axis, and \(\Delta_{Y}\) is the distance between the wireless sensor's antennas along the y-axis. 

For an arbitrary number and positioning of wireless nodes, the gradient calculation can be generalized as in \cite{parasuraman2014multi}.
\begin{eqnarray}
\begin{aligned}
\color{black}
    g_x =  \frac{S_{x_c+\lambda,y_c-\delta} - S_{x_c-\lambda,y_c-\delta}}{2\Delta_{X}} + \frac{S_{x_c+\lambda,y_c+\delta} - S_{x_c-\lambda,y_c+\delta}}{2\Delta_{X}} \\
\color{black}
    g_y = \frac{S_{x_c+\lambda,y_c-\delta} - S_{x_c+\lambda,y_c+\delta}}{2\Delta_{Y}} + \frac{S_{x_c-\lambda,y_c-\delta} - S_{x_c-\lambda,y_c+\delta}}{2\Delta_{Y}}
    \label{eqn:gradient-general}
\end{aligned}
\end{eqnarray}
Here, $\lambda$ = 0.5\(\Delta_{X}\),  $\delta$ = 0.5\(\Delta_{Y}\), and $S_{x_c,y_c}$ is the RSS value from a node measured at the current path location $(x_c,y_c)$.

We then calculate the CDOA from the gradients calculated using Eq.~\eqref{eqn:gradient-general}.
\begin{equation}
    CDOA_l=\arctan(\frac{g_y}{g_x})
    \label{eqn:doa}
\end{equation}
The formula provides the CDOA of the wireless signal at a position along the path \(l\) using the RSS gradient.
We can then suppress the noise of the calculated CDOA by using the exponentially weighted moving average.

We employ a Gaussian probability model on the wireless signal CDOA estimates to calculate the weights of each random particle in the PF. 
Similar to the work in \cite{li2015modified} that uses acoustic signals, this probabilistic model will weigh the quality of signals sensed by each node from \(N\) and ultimately produce an accurate robot location estimate through the PF.

The absolute error between Actual CDOA for all wireless sensors at a potential candidate position \(l\) of mobile robot with coordinates \((x_c, y_c)\) to the perceived  ${CDOA}$ values for each sensor calculated for each particle.
Later, we use the Gaussian probability formula (similar to \cite{parashar2020particle}) on this error to calculate the probability of the \(i_{th}\) candidate location of the particle \( q_i = (x_i, y_i, w_i), i \in (1, ..n)\), where \(n\) is the number of samples in the PF spread in the bounded region with the resolution of $\Re$ and \(w_i\) is the weight of each particle calculated over a set of previous path samples as:
\begin{equation}
    P_l(q_i) = \prod_{k=0}^{M-1} \bigg[\frac{1}{\sigma\sqrt{2\pi}}e^{-\frac{(err_l^{k})^2}{2\sigma^2}}\bigg] ,
    \label{eqn:probability}
\end{equation}
\textcolor{black}{where $ P_l(q_i)$ is the probability density function for the position $l$ of candidate particle $i$ in PF, $M$ is the number of the previous samples considered, $\sigma$ is the standard deviation of error for all previous samples, and $(err_l^{k})^2$ is the CDOA error for sample $k$. Eq.~\eqref{eqn:probability} provide the probability for particle $q_i$ considering $M-1$ previous samples.}

There is an intrinsic angular inaccuracy in each CDOA degree that is analyzed. \(\sigma\) represents the fluctuation (deviation) of this error, which is anticipated to be known because we know the correctness of the technique used to assess the CDOA. We use the product of the Gauss likelihood of CDOA error over \(M - 1\) prior robot positions (imitating geographically scattered samples) so that the sifted CDOA from earlier path locations can be used similarly as readings from many sensors. 

Next, we discuss how the CDOA estimation is integrated with a probabilistic framework to achieve localization using the DOA information. 

\subsection{CDOA-PF}
The CDOA probability from Eq.~\eqref{eqn:probability} is used to calculate the weights of the particles in the PF, which is then employed in the resampling procedure in the next PF step (iteration).
\begin{equation}
    w_i  \propto P_l(q_i)
    \label{eqn:weights}
\end{equation}
The EM and PF provide initial hypotheses with a uniform sampling of probable robot locations across the environment using a constraint around the present robot location. The Gaussian probability is determined for each particle, the signal source. The particles are subsequently given weights that are proportional to their likelihood, and the weights \(w_i\)
are normalized as $w_i^*(q_i)=\frac{w_i(q_i)}{\sum_{i=0}^{n-1}w_i(q_i)}$.

This normalized weight determines the likelihood of regenerating a particle in the next iteration. The particle with the highest weight (softmax) best gauges the robot's location. This process is repeated, and the particles eventually converge on the location estimates. It is worth noting that the PF is iterated for each new estimation tuple. 
{Alg.~\ref{alg:pf-cdoa} depicts a pseudo-code of particle filtering to estimate the location from the CDOA efficiently. The CDOA-PF algorithm combines a Particle Filter technique with a Cross Difference of Arrival. The algorithm incorporates transition models and CDOA computations and iteratively updates particles that reflect the robot's potential positions. The algorithm calculates the robot's position by resampling particles according to their associated probabilities and doing so until the end of the robot's trajectory.}

\subsection{CDOA-EM}
\textbf{Expectation Maximization (EM)} is a grid-based localization that uses an explicit, discrete representation for the probability of all positions in the state space. We represent the environment by finite discrete state spaces (Grids). The algorithm updates the probability of each state of the entire space at each iteration. Use a fixed decomposition grid by discretizing each CDOA: \((x, y, \theta)\). For each location \(x_i = [x,y,\theta]\) in the configuration space: determine probability \(P(x_i)\) of the robot being in that state. Then, it chooses the state with the highest probability. 
This approach resembles the EM method in \cite{measRSS}.
{Alg.~\ref{alg:em-cdoa} depicts the procedure of the EM approach that can be used to estimate the location from the CDOA efficiently. The CDOA-EM algorithm employs Expectation Maximization and Cross Difference of Arrival. It iteratively updates the robot's position using transition models and CDOA calculations, populating a grid with positions as samples. The program determines the maximum probability grid position by converting CDOA into Gaussian probability and predicting the robot's location until the end of the trajectory.}

\begin{figure}[t]
\centering
 \includegraphics[width=0.99\linewidth]{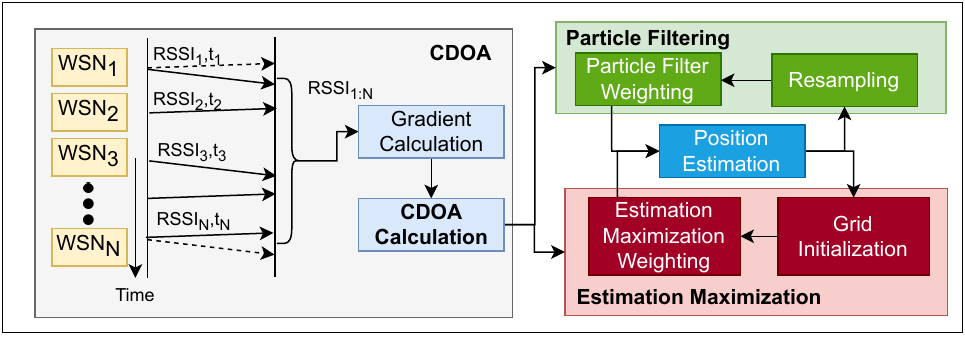}
 \caption{\textcolor{black}{System architecture of the CDOA-based robotic node localization approach using particle filter and estimation maximization.}}
 \label{fig:architecture}
\end{figure}

\subsection{Summary of the system architecture}
\textcolor{black}{Fig.~\ref{fig:architecture} delineates the system architecture of proposed CDOA-based indoor localization in which WSNs collaboration is shown in \textit{CDOA} block where all nodes share RSSI and perform gradient calculation using Eq.~\eqref{eqn:gradient}, which was further converted into CDOA using Eq.~\eqref{eqn:doa} as mentioned in Alg.~\ref{alg:cdoa}. CDOA with robots' position estimation is used for PF resampling as discussed in Alg.~\ref{alg:pf-cdoa}, and similarly, for estimation maximization as mentioned in Alg.~\ref{alg:em-cdoa}. These estimates will further be used for state estimation of the robot as the robot moves along a trajectory.}

\begin{algorithm}[t]
  $P_r$ = [] \% Initial list of particles in the PF \;
  \While{end of trajectory}{
  draw sample \(x_t\) from transition model \(P(X_{t}|x_{t-1},z_{t-1})\)\;
   CDOA calculation using Algorithm \ref{alg:cdoa}\;

   $\omega$ = [] \% weights\;
   \For{\(x_t\) in $P_r$}{
   add \(P(z_t|x_t)\) to $\omega$\;
   }
   $R_r$ = [] \% re-sampling\;
   \For{i=1 to n (number of particles)}{
   Choose p = $P_r$[i] of probability $w_i(x_i) \in \omega$\;
   add p to $R_r$\;
    \(x_l^* = x_i \in w_i(x_i)\) is \(max(w_i)\)\;
   }
   $P_r = R_r$\;
  }
  \caption{CDOA-PF: Particle Filter Over CDOA for Mobile Robot Localization}
  \label{alg:pf-cdoa}
\end{algorithm}

\begin{algorithm}[t]
  Grid dimensions and resolution is given\;
  $E_r$ = [] \% Populate full-resolution grid positions as samples for EM\;
  \While{end of trajectory}{
  draw sample \(x_t\) from transition model \(P(X_{t}|x_{t-1},z_{t-1})\)\;
   CDOA at $x_t$ calculation using Algorithm \ref{alg:cdoa}\;
    $\omega$ = [] \% weights\;
   \For{each grid position in $E_R$}{
   Convert CDOA into $w_{E_r}(x_t) \in \omega$ Gaussian Probability using (Eq.~\eqref{eqn:probability}\ for $E_r$[t]\;
    \(x_l^* = x_t \in w_{E_r}(x_t)\) is \(max(w_{E_R})\)\;
   }
  }
  \caption{CDOA-EM: Expectation Maximization Over CDOA for Mobile Robot Localization}
  \label{alg:em-cdoa}
\end{algorithm}

 
\section{Theoretical Analysis}
\textbf{Assumption 1} Given the locations of static wireless nodes $N$, their range $R$ in the wireless sensor network, and the position of the robot node $x$; the following observation can be made about the CDOA estimation probability: CDOA is independent of the previous observations, i.e.,
\begin{equation}
    P(CDOA_t| CDOA_{t-1},N,R,x) = P( CDOA_{t}|N,R,x)
\end{equation}

\textbf{Assumption 2} Given $X$ as a set of samples in PF and $\Re$ as the resolution of spread. We assume that samples spread randomly in the space with the given resolution spread, greater than the centroid of converged samples in PF. i.e.,
\begin{equation}
    \frac{\sum_{i=0}^{n-1}(X_i)}{n}\leq \Re
\end{equation}

\begin{lemma} 
The error in the location estimation depends upon the cumulative noise percentage $\mathcal{N\%}$ of RSSI from each wireless sensor node in wireless sensor network.
\label{lemma:error}
\end{lemma}
\begin{proof}
First, We prove the relation between the error in CDOA estimation of a robot at a candidate position $l$ with cumulative noise ~$\mathcal {N\%}$ in RSSI from WSNs.

Let $\eta_x = {\eta_{x,1},\eta_{x,2}, ..., \eta_{x,j},}$ and $\eta_y = {\eta_{1,y},\eta_{2,y}, ..., \eta_{k,y},}$ are noise values of nodes in the horizontal and vertical axis in wireless sensor network respectively. Based on Eq.~\eqref{eqn:gradient}, the uncertainty values for $g_x$ and $g_y$ are $\sum_{i=1}^{j}(\eta_{x,i})$ and, $\sum_{i=1}^{k}(\eta_{i,y})$ respectively. Hence, calculated CDOA at position $l$ has a cumulative percentage error of $g_x$ and $g_y$.
Next, we note that location estimation is the soft-max of weights of $n$ samples in PF; hence the error in location estimation depends upon the cumulative percentage error of maximum weight in PF, which is further dependent upon the error of CDOA estimation using RSSI in wireless sensor network.

Moreover, as shown in \cite{penna2011bounds}, the location estimation error variance estimated using the DOA metric can scale linearly with the DOA estimation variance and quadratically with the target distance (see Lemma~\ref{lemma:coverage}).
\end{proof}

\begin{theorem}
(\textbf{Convergence}) The CDOA-based localization output will converge to the actual position of the robot in finite iterations. Let $x_i=\tau x_{i-1} + \epsilon$, where $i$ is the current iteration, $\tau \in X$ and $\epsilon \approx \Re$ (resolution). We can claim that the infinite sequence, ${x_i}^{\infty}_{i=1}$ has an approximate solution in a finite number of iterations.
\label{theorem:convergence}
\end{theorem}
\begin{proof}

Let $x_1, x_2, ..., x_n$ be the converged particles in the PF solution, and let $c$ be the centroid of these particles. The particle filter algorithm iteratively updates the particles according to the likelihood function and the prior distribution. After a finite number of iterations, the particles converge to a stable solution that approximates the true position.

We can represent the particle filter algorithm as an iterative process, where the $k$-th iteration is represented by the function $g_k(x_{k-1}) = x_k$. We can prove that the sequence of iterations $x_k$ converges to a fixed point $x^*$ of the function $g(x) = \lim_{k \rightarrow \infty} g_k(x)$

Now, let $x$ be any candidate particle with a minimum difference from each of the converged particles $x_1, x_2, ..., x_n$. By the definition of the centroid, we know that $\sum_{i=1}^{n} | x_i - c|$ is minimized. Therefore, $| x - c|$ is also minimized, which means that $x$ is closest to $c$
Therefore, we can say that position estimation can be obtained with uncertainty ~$\mathcal{N\%}$ using Lemma 1, and the proposed algorithm converges over time and results in accurate location estimation.

A similar proof can be derived to guarantee the convergence of the CDOA-EM localization output.
\end{proof}

\begin{lemma}
The accuracy of location estimation depends upon the resolution spread $\Re$, i.e., 
\begin{equation}
    \operatorname*{argmin}_{i=0}^{n-1} |x-X_i|\leq \Re
\end{equation}
\label{lemma:spread}
\end{lemma}
\begin{proof}
According to Theorem 1, we have proven that the PF/EM solution converges to the actual position after finite iterations. Let $X = {X_0, X_1, \ldots, X_{n-1}}$ be a set of $n$ known locations and let $x$ be an unknown location that we wish to estimate. We are given that the accuracy of location estimation depends on the resolution spread $\leq\Re$ based on Assumption 2.
To prove this lemma, we will show that for any $x$, the distance between $x$ and its closest known location $X_i$ is less than or equal to $\Re$.

We have demonstrated that Theorem 1's statement is true and that the state estimation converges to the actual position after a finite number of iterations. Now, define $X$ as a collection of $n$ known locations, where $X = {X_0, X_1, \ldots, X_{n-1}}$ is the set, and define $x$ as an unknown location that we want to estimate. By Assumption 2, we are informed that the resolution spread $\leq\Re$ determines how accurately we can estimate the position.
To demonstrate this lemma, we shall demonstrate that for any $x$, $x$'s distance from its nearest known position $X_i$ is less than or equal to $Re$.

First, we note that for any $i \in {0, 1, \ldots, n-1}$, the distance between $x$ and $X_i$ is given by $|x-X_i|$. Therefore, the closest known location $X_j$ to $x$ is the one that minimizes this distance, i.e. ${argmin}_{i=0}^{n-1} |x-X_i| = X_j$.
We want to show that $|x-X_j| \leq \Re$. To do this, we use the definition of ${argmin}{i=0}^{n-1} (x-X_i)\leq \Re$. We know that ${argmin}{i=0}^{n-1} (x-X_i)\leq \Re$ means that for any $i \neq j$, we have $|x-X_j| \leq |x-X_i| \leq \Re$. It follows that $|x-X_j| \leq \Re$ proves the lemma.
\end{proof}

\begin{lemma}
\textbf{(Coverage)} With a minimum of 4 WSNs or IoT nodes in the network available for collaboration with a sensing range of $r$ each, the CDOA localization method's maximum coverage area is $\frac{r^2}{2}$, as long as the robot node to be localized is within the boundary of the collaborating wireless nodes.
\label{lemma:coverage}
\end{lemma}
\begin{proof}
Assume that the four WSNs or IoT nodes $A$, $B$, $C$, and $D$ are situated at the four corners of a square region (see Fig.~\ref{fig:overview} and \ref{fig:coverage}). Assume that each side of the square is $s$ in length. 
The maximum distance between the robot $i$ and any of these nodes must be less than or equal to the diagonal of the square, which is $\sqrt(2) s$, because $A$, $B$, $C$ and $D$ are located at the corners of the square.
As a result, the sensing range $r$ must be at least a distance of $\sqrt(2) s$ to connect the robot node with other supporting nodes in the network (i.e., $r \geq \sqrt{2} s$, meaning $s \leq \frac{r}{\sqrt{2}}$). 
The square region, which forms the outer boundary workspace of the robot node, has an area of $s^2$.
Therefore, the reliable area of the region that the CDOA approach can cover for localization with nodes having a sensing range of $r$ is $\geq \frac{r^2}{2}$ as long as the robot is inside the boundaries of the four cooperating nodes.
Moreover, expressing the coverage area $A(r)$ as a function of $r$, we can take its derivative as 
\begin{equation}
    \frac{\partial A(r)}{\partial r} = \frac{\partial (\frac{r^2}{2})}{\partial r}  = 1
\end{equation}
We can see that this derivative is always positive and indicates the linearity of the coverage area concerning the number of nodes and the sensing range. 
\end{proof}

\begin{figure}[t]
    \centering
    \includegraphics[width=0.9\linewidth]{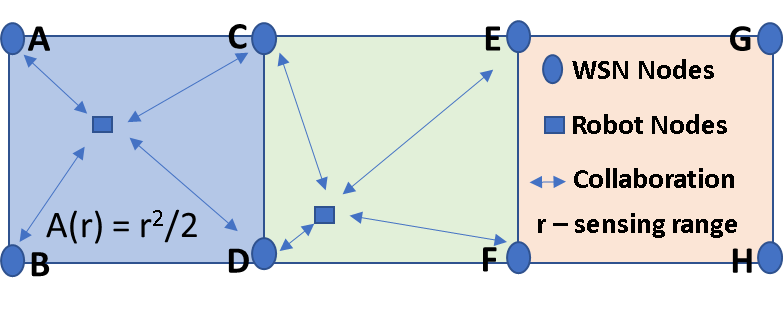}
    \vspace{-4mm}
    \caption{Depiction of coverage area for CDOA localization.}
    \label{fig:coverage}
    \vspace{-4mm}
\end{figure}
\textbf{Remark 1} Minimum of three nodes (instead of 4 nodes) can be sufficient as per the one-way finite difference equation for node collaboration to estimate DOA (see \cite{parasuraman2013spatial}), as long as all the nodes encompass the convex hull of the area boundary.

\textbf{Remark 2} If more nodes are available to collaborate than the minimum number of nodes, this allows exploiting redundancy in the CDOA estimation, providing robustness and accuracy advantages.

\begin{lemma}
\textbf{(Coverage Generalization)} The minimum area coverage mentioned in Lemma~\ref{lemma:coverage} can be generalized to a rectangular region with an area of $r^2.\frac{k}{k^2 + 1}$, where $k$ is the width factor of the length and width of the new rectangular area, and $r$ is the sensing range. The maximum coverage is achieved when the length is equal to the width (square region).
\label{lemma:expansion}
\end{lemma}
\begin{proof}
Trivial. Given are the dimensions of a rectangle: length $l$, width $w$, and aspect ratio $k$ (the width factor that makes $w = kl$). The sensing range $r$ must be greater or equal to the largest diagonal of the rectangle. That is, $r \geq \sqrt(l^2 + w^2)  \geq l^2 (1 + k^2)$. Therefore, the coverage area is expressed as $l^2 \leq r^2.\frac{k}{1+k^2}$. 
Moreover, it is trivial to observe that the width and length must be identical in order to cover the most area, i.e., when $k = 1$. This relates to a square region. 
\end{proof}

\textbf{Remark 3} The CDOA estimation and coverage area can be generalized to an arbitrary polygonal shape of the boundary of the localization workspace as long as the convex hull of the boundary nodes can be defined. For instance, the CDOA can be estimated using gradient estimation algorithms as proposed in \cite{han2009access,verma2018direction} when the WSNs are distributed in the workspace without a specific geometric pattern.

\begin{lemma}
        \textbf{Number of Collaborating Nodes} A minimum of $2n + 2$ number of nodes are required to cover an area of $n\frac{r^2}{2}$, for wireless nodes with sensing range $r$.
\label{lemma:nodes}
\end{lemma}
\begin{proof}
Let $A$ be the unit area of the square that can be covered by 4 WSNs (using the result of Lemma~\ref{lemma:coverage}. To extend this coverage beyond this unit area with a scaling factor of $n$, we would need $2n$ more nodes, which require replicating the square region $n$ number of times as shown in Fig.~\ref{fig:coverage}. With $2n+2$ nodes, the maximum area of coverage then becomes $nA = n\frac{r^2}{2}$. For example, for a unit area $A$ and $n=1$, we need four nodes (nodes A-D); with double the coverage area, we need a minimum of 6 nodes (nodes A-F). Therefore, the number of nodes required for CDOA scales linearly with the coverage requirement.
\end{proof}

\begin{theorem}
\textbf{(CDOA Scalability)} Combining the results of Lemmas~\ref{lemma:expansion} and \ref{lemma:nodes}, the approximate linear relationship between the number of nodes $n$ and the coverage area $A$ for localization using the CDOA approach can be expressed as $A \approx cn + d$, where $c$ and $d$ are constants.
\end{theorem}
\begin{proof} We will prove this theorem using the mathematical induction principle.
Considering Eq.~\eqref{eqn:gradient}, which uses four wireless nodes ($n=4$) and proposed approach able to find position estimation in the bounded region, given by algorithm Alg.\ref{alg:cdoa}, also, in this case, the coverage area is $A_4 = 4c+d$ where $c$ and $d$ both are constants, which satisfies the theorem (see Lemma~\ref{lemma:expansion}).
Next, we will assume that the theorem holds for some arbitrary value of $n=k$ where $k>4$. In other words, we assume that $A_k \approx ck + d$.
Now, we need to prove that the theorem also holds for $n=k+1$. The coverage area for $n=k+1$ nodes is given by $A_{k+1} = A_k + \Delta A$, where $\Delta A$ is the increase in coverage area due to the addition of one more node.
Since the CDOA approach is based on the DOA of RSSI, it can be assumed that the increase in coverage area due to adding one more node is approximately proportional to the existing coverage area. This means that $\Delta A \approx \alpha A_k$, where $\alpha$ is a constant (see Lemma~\ref{lemma:nodes}).
Substituting the value of $A_k$ from the induction hypothesis, we get $\Delta A \approx \alpha (ck+d)$. This means that the total coverage area for $n=k+1$ nodes is given by:
\begin{eqnarray}
A_{k+1} \approx ck+d + \alpha (ck+d) = (c+\alpha)k + (d+\alpha d). \\
A_1 = d ; \; A_{k+1} \approx (c+\alpha)k + (d+\alpha d) ; \; A_n \approx cn + d 
\end{eqnarray}

Thus, we have proved that the theorem holds for $n=k+1$. Since we have established the base case and have shown that the theorem holds for all $n=k+1$ if it holds for some arbitrary value of $n=k$, we can conclude that the theorem holds for all $n \geq 4$.
\end{proof}

\textbf{Computational Complexity} 
Let $n$ be the number of samples (or particles) in the EM (or PF), and $N$ be the number of cooperative nodes in the wireless sensor network.
At each step, the robot needs to find pose estimation based on our proposed WSN cooperative localization algorithm, which involves the following steps:
\begin{itemize}
    \item EM and PF initialization with random particles: $O(n)$.
    \item Weight transfer to new PF for each sample in $O(n)$.
    \item Robot-WSN collaboration in an open time window of $\alpha$ for sharing and receiving wireless signals from $N$ nodes in $O(\alpha N)$ time.
    \item Sample weight calculation based on CDOA in $O(n)$.
    \item Finding the soft max from particle weight distribution as the best pose estimate in $O(n/2)$.
\end{itemize}
Therefore, one iteration is $O(n \times \alpha(N))$ where $\alpha$ and $N$ are constant values and have a low impact on overall computation complexity; hence, overall time consumed by the proposed algorithm would be $O(n)$.

\section{Experimental Validation}
We implemented our approach and compared the performance with relevant recent methods from the literature (see Sec.~\ref{sec:SOTA}.
We performed extensive experiments through simulations (Sec.~\ref{sec:sim-exp}), real-world datasets (Sec.~\ref{sec:dataset}), and real robot in-house experiments (Sec.~\ref{sec:hardware_testbed}) to verify and validate the performance of the proposed localization in terms of accuracy measured through the Root Mean Squared Error (RMSE) and efficiency measured through the Time Per Iteration (TPI) metrics.
In each experiment, we made 100 trials and averaged the localization error over all trials as the distance between the predicted and the actual positions (ground truth). 
The experiment settings shown in Table~\ref{tab:exp-config} show diverse settings under which we evaluate the proposed method.

\begin{table}[t]
\begin{center}
\caption{Experiment Configurations}
\label{tab:exp-config}
\vspace{-4mm}
\resizebox{\linewidth}{!}{
\begin{tabular}{|p{13em}|c|c|c|c|}
\hline
\textbf{Simulation Parameters}&\multicolumn{4}{|c|}{\textbf{Experiment Basis}} \\
\cline{2-5} 
& \textbf{Simulations} & \textbf{Dataset1} & \textbf{Dataset2} &\textbf{Hardware}\\   
\hline
Space Dimensions & $6 \times 6 $ & $4 \times 4 $ & $10 \times 10$ & $2.34 \times 1.75$ \\
\hline
Resolution of CDOA-EM (ppi) & 0.05 & 0.05 & 0.1 & 0.05\\
\hline
Resolution of CDOA-PF (ppi) &0.08 & 0.1 & 0.5 & 0.08 \\
\hline
\end{tabular}
}
\end{center}
\end{table}

\subsection{Comparison with the State-of-the-Art (SOTA)}
\label{sec:SOTA}

To validate the results of our proposed approach, we implemented the five model-based solutions from the recent literature: 1) \textbf{Trilateration} \cite{yang2020trilateration}, 2) \textbf{WCL: Weighted centroid localization} \cite{weightedCentroid2014survey}, 3) \textbf{D-RSSI: differential RSSI-based localization} \cite{podevijn2018comparison}, 4) \textbf{I-RSSI: Improved RSSI based localization} \cite{Xue2017improvedrssi}, 5) \textbf{PF-EKF: Particle filter - Extended Kalman Filter} \cite{zafari2018baysianfiltering}, {6) \textbf{SBL-DOA: Sparse Bayesian Learning applied over the Direction of Arrival} information \cite{wang2019assistant}.}
Our CDOA-PF approach applies the  Gaussian probability over CDOA for a fixed number of sampled particles, and in our CDOA-EM, the solution is extensively searched through each and every grid of the workspace with a fixed resolution (the EM implementation is similar to the method in \cite{measRSS}).
More details on each of these methods are included in Appendix~\ref{sec:appendix-methods}.

\begin{figure*}[ht]
    \centering
\includegraphics[width=0.98\linewidth]{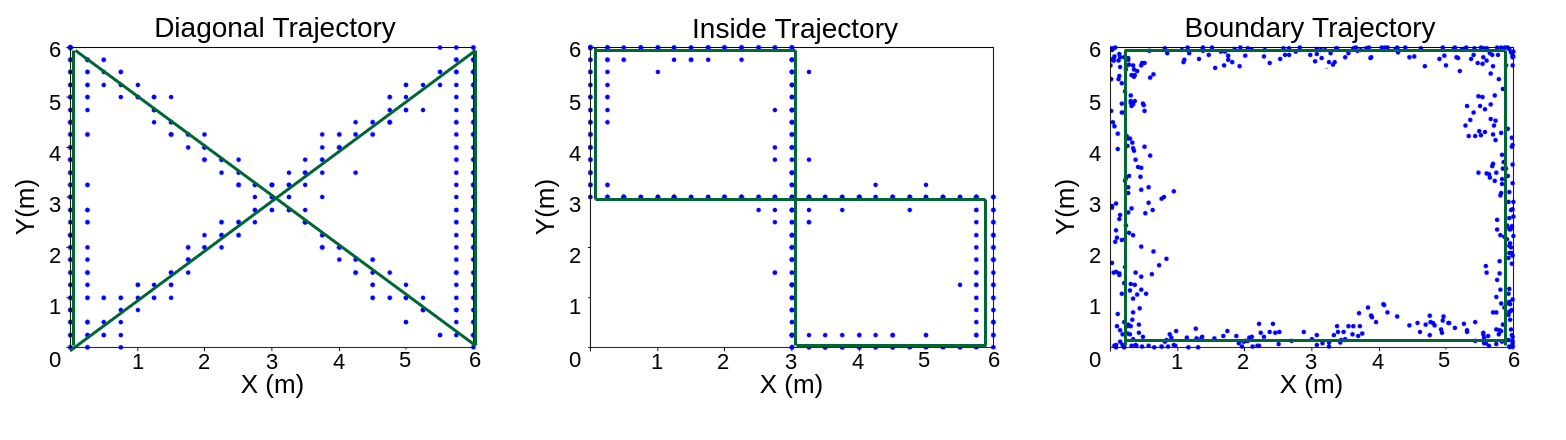}
\vspace{-6mm}
\caption{The test sample locations in the simulation experiments are shown here. The plots show the robot's actual trajectory (where the signal sample was taken) in solid lines and the predicted locations in scattered dots.}
    \label{simulation_trajectory}
\end{figure*}

\subsection{Numerical Analysis with Simulations} 
\label{sec:sim-exp}
We simulated four WSNs distributed on the corners of the simulation workspace, with the robot's initial position being their center (Fig.~\ref{fig:overview}). We simulated three different trajectories for robot motion: Boundary (left), Cross Coverage (center), Diagonal (right), and the scale of the workspace is 6m x 6m, as shown in Fig.~\ref{simulation_trajectory}.
Through these paths, we cover all potential positions within the bounded region.

We measure the RSS value of each WSN for all positions along the robot's path. The estimated RSS based on the log-normal radio signal fading model is computed as:
\begin{equation}
RSSI = A - 10\times \eta \times \log_{10}(d) ,
\end{equation}
where $\eta$ denotes the path loss exponent, which varies between 2 (free space) and 6 (complex indoor environment), \(d\) denotes the distance from Robot \(R\) to the node \(N\), and \(A\) denotes received signal strength at a reference distance of one meter.
We used this setup to perform experiments and validate the accuracy of localization techniques for different noise conditions on the measurements simulated through a zero-mean Gaussian noise varying from 1 to 4 dBm variance.
The path loss exponent $\eta$ is set to 3 in our simulations to present a reasonable indoor environmental channel in our simulations. \textcolor{black}{The simulation effectively mimics noise in RSSI by incorporating factors like distance, signal frequency, and environmental conditions. Moreover, with varied noise levels, it aptly represents diverse real-world scenarios, substantiating its representativeness in our experiments.}

\begin{figure}[t]
\centering
 \includegraphics[width=0.47\linewidth]{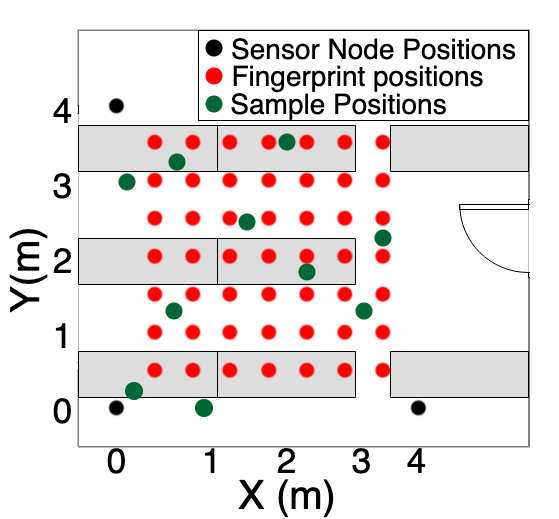}
 \includegraphics[width=0.51\linewidth]{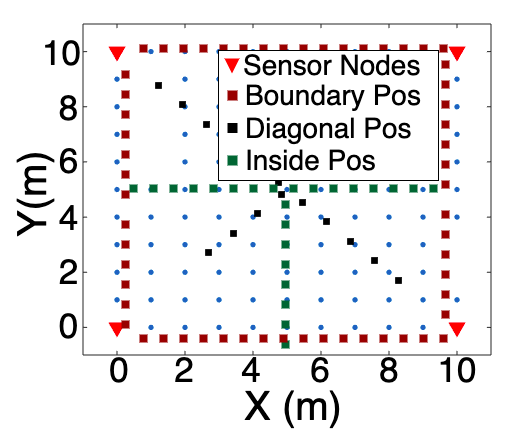}
 \vspace{-4mm}
 \caption{\textcolor{black}{Locations of the RSSI sample points from the real-world datasets 1 (Left) and 2 (Right) discussed in Sec.~\ref{sec:dataset}.}}
 \label{fig:datasets}
\end{figure}

\subsection{Real-world Datasets}
\label{sec:dataset}
We used two different publicly available real-world RSSI datasets on indoor localization. 
\textcolor{black}{See Fig.~\ref{fig:datasets} for the illustration of the workspace of the datasets along with the positions of the WSNs and locations where the RSSI samples are measured.}

\textbf{Dataset 1\footnote{\url{https://github.com/pspachos/RSSI-Dataset-for-Indoor-Localization-Fingerprinting}}} Provides RSSI values for three wireless technologies; BLE, Zigbee, and Wi-Fi in a 4m x 4m room. We have used the data for Scenario 1 of this dataset as it relates to our approach to the geometric positioning of anchor nodes. In this scenario, a room of 6.0 x 5.5 m was used as the experimental testbed. 
All transmitting devices were removed from the surroundings to establish a transparent testing medium where all devices could communicate without interference. 
The transmitters were spaced 4 meters apart in the shape of a triangle. 
The fingerprint and test points were obtained with a 0.5 m distance between the transmitters in the center. 
The database would be made up of 49 fingerprints due to this. Ten test points were chosen at random for testing. 
We have arranged the fingerprinting dataset in such a way that it makes a trajectory in the region. 

\textbf{Dataset 2\footnote{\url{https://ieee-dataport.org/documents/multi-channel-ble-rssi-measurements-indoor-localization}}} Provides RSSI values for the three regions of varying ranges in the bounded area: diagonal, boundary, and inside in a 10m x 10m room. Four anchors took RSSI measurements while receiving messages from a single mobile node, delivering advertisement and extended advertisement messages in all BLE channels (both primary and secondary advertisement channels). Four anchors were placed in the corners of a 10 x 10 m office area (no considerable impediments).
We have compared the results for different communication channels under different regions in the bounded area.

\begin{figure*}[ht]
    \centering
\includegraphics[width=0.3\linewidth]{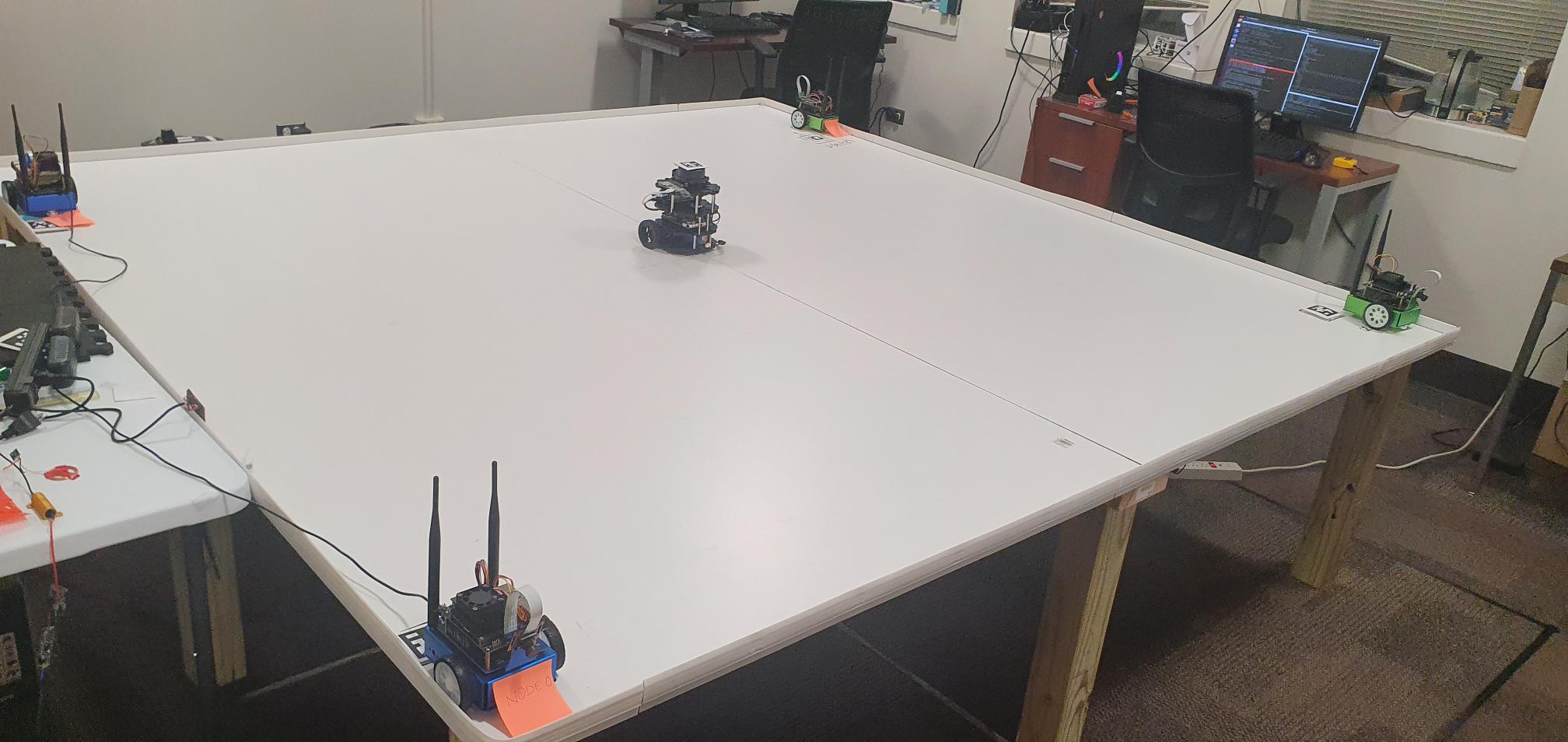}
\includegraphics[width=0.69\linewidth]{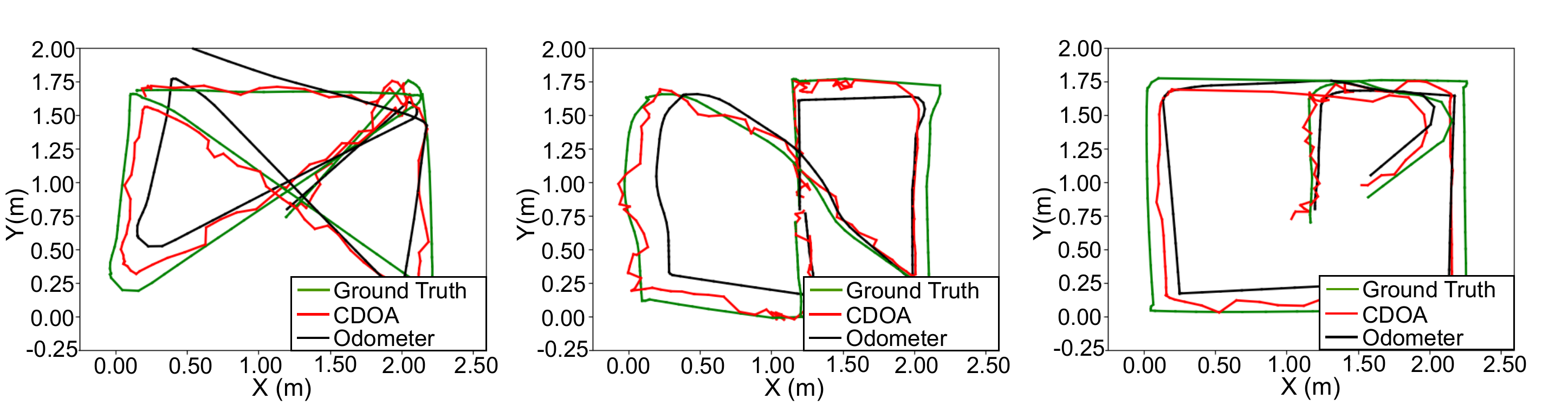}
\vspace{-4mm}
\caption{{Hardware experiment setup and samples of the output robot's trajectories (diagonal, inside, boundary) in the real robot experiments (See Sec.~\ref{sec:hardware-results}).}}
    \label{fig:hardware_trajectory}
\end{figure*}

\subsection{Real Robot Hardware Experiments}
\label{sec:hardware_testbed}
To validate the practicality of the algorithm for real-world scenarios, we performed hardware experiments on a testbed of dimensions $2.34 m \times 1.75 m$, with the ceiling-mounted camera for visual localization as ground truth; we have mounted a wireless access point of power 20dBm with a 2.4Ghz frequency over the top of a Turtlebot3 mobile robot.

Fig.~\ref{fig:hardware_trajectory} shows the experimentation test-bed and sample trajectories with the Turtlebot3 robot at the center of the WSNs on the four corners. 
In each trial, we drove the robot remotely and recorded RSSI and ground truth position with an overhead camera-based fiducial marker (AprilTag \cite{wang2016apriltag}) tracking). Robot Operating System (ROS \cite{quigley2009ros}) has been employed for inter-node communication, as ROS is the de-facto software framework used in the robotics literature. The ROS master runs a service on the experimental robot to receive perceived signal strength from all connected nodes through a synchronization service. 
\textcolor{black}{The WSN nodes operate at a 10Hz rate to calculate the RSSI and publish these values to the ROS topics.}
Five trials have been conducted for each of the three different trajectories.

\section{Results}
\label{sec:results}

Table.~\ref{tab:results} comprehensively presents the overall DOA-only localization performance, time complexity, and efficiency results of the proposed approach compared to the SOTA approaches in simulation, real-world datasets, and real robot hardware experiments. 
For time complexity, the variables $n$, $S$, and $r$ represent the number of particles in PF, the size of the grid in EM, and the resolution of grids, respectively.
Looking at the TPI (efficiency) metric, the model-based approaches such as Trilateration and WCL are the fastest because of the low computational requirement they need for every new sample. However, the CDOA-PF is comparable to the model-based methods in terms of real-time computational tractability, allowing the possibility for instantaneous localization. 
In general, the CDOA methods performed significantly better than the baselines in terms of higher accuracy and reasonable efficiency. 
More details on the results are discussed below.

\begin{table*}[ht]
\begin{center}
\caption{Overall performance results with statistics of various experiments from simulations, datasets, and hardware trials.}
\label{tab:results}
\resizebox{0.75\linewidth}{!}{
\begin{tabular}{|p{8em}|c|p{4em}|c|c|c|c|}
\hline
\textbf{Algorithm }& \textbf{Complexity} & \textbf{Average TPI (ms)} & \multicolumn{4}{|c|}{\textbf{Localization Error RMSE (m)}} \\
\cline{4-7}
&  & &\textbf{Simulations} & \textbf{Dataset1} & \textbf{Dataset2} &\textbf{Hardware}\\ 
\hline
Trilateration &$O(1)$ & \textbf{82$\pm$18} & $1.22\pm0.56$ & $2.70\pm1.78$ & $3.73\pm2.27$ & $1.63\pm0.85$ \\
\hline
WCL &$O(1)$ & $83\pm27$ & $2.47\pm0.84$ & $3.94\pm0.95$ & $4.93\pm3.26$ & $2.54\pm1.21$ \\
\hline
D-RSSI &$O(1)$& $86\pm18$ & $0.55\pm0.13$ & $2.18\pm1.26$ & $2.92\pm1.52$ & $1.53\pm0.72$ \\
\hline
I-RSSI &$O(1)$& $94\pm21$ & $0.42\pm0.09$ & $1.98\pm0.86$ & $2.43\pm1.21$ & $1.11\pm0.57$  \\
\hline
PF-EKF &$O(n)$ & $111\pm34$ & $0.91\pm0.18$ & $2.34\pm1.64$ & $3.12\pm1.78$ & $1.74\pm0.91$ \\
\hline
{
SBL-DOA} & {$O(n)$} & {$127\pm42$} & {$0.16\pm0.06$} & \textcolor{black}{$1.81\pm0.69$} & \textcolor{black}{$1.72\pm0.48$} & \textcolor{black}{$0.54\pm0.21$} \\
\hline
CDOA-EM (Ours) &$O(S*r)$& $270\pm99$ & \textbf{0.13$\pm$0.04} & $1.58\pm0.53$ & \textbf{1.66$\pm$0.73} & $0.34\pm0.08$ \\
\hline
CDOA-PF (Ours) &$O(n)$ & $102\pm25$ & $0.15\pm0.05$ & \textbf{1.17$\pm$0.48} & $1.67\pm0.83$ & \textbf{0.12$\pm$0.16} \\
\hline
\end{tabular}
}
\end{center}
\end{table*}

\subsection{Simulation Results}
\label{simulation}

\begin{figure}[t]
    \centering
 \includegraphics[width=0.98\linewidth]{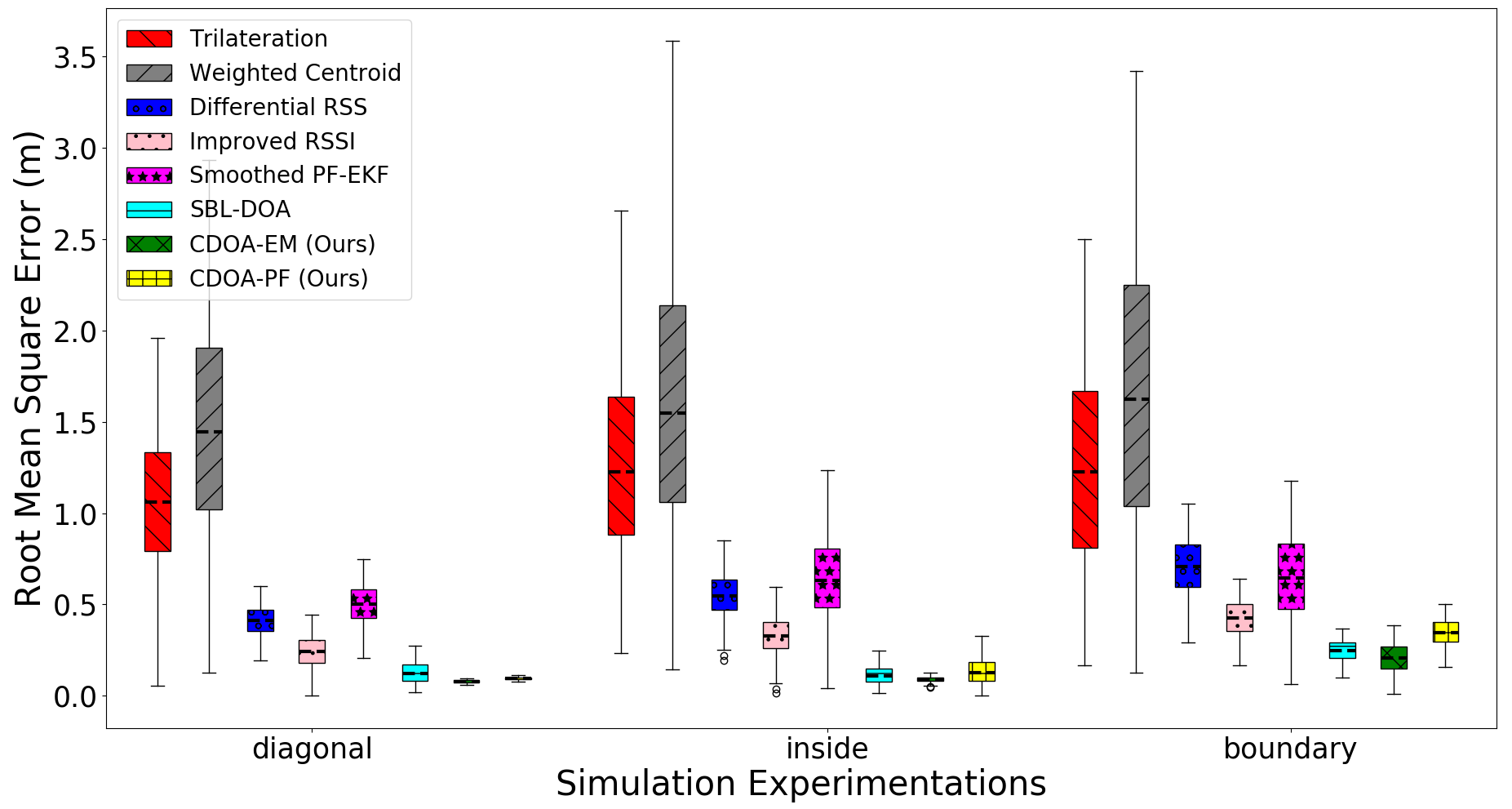}
 \vspace{-2mm}
 \caption{Comparative localization performance in the simulation environment.}
 \label{fig:sim-results}
\end{figure}

\begin{figure}[t]
\centering
\begin{center}
 \includegraphics[width=0.98\columnwidth]{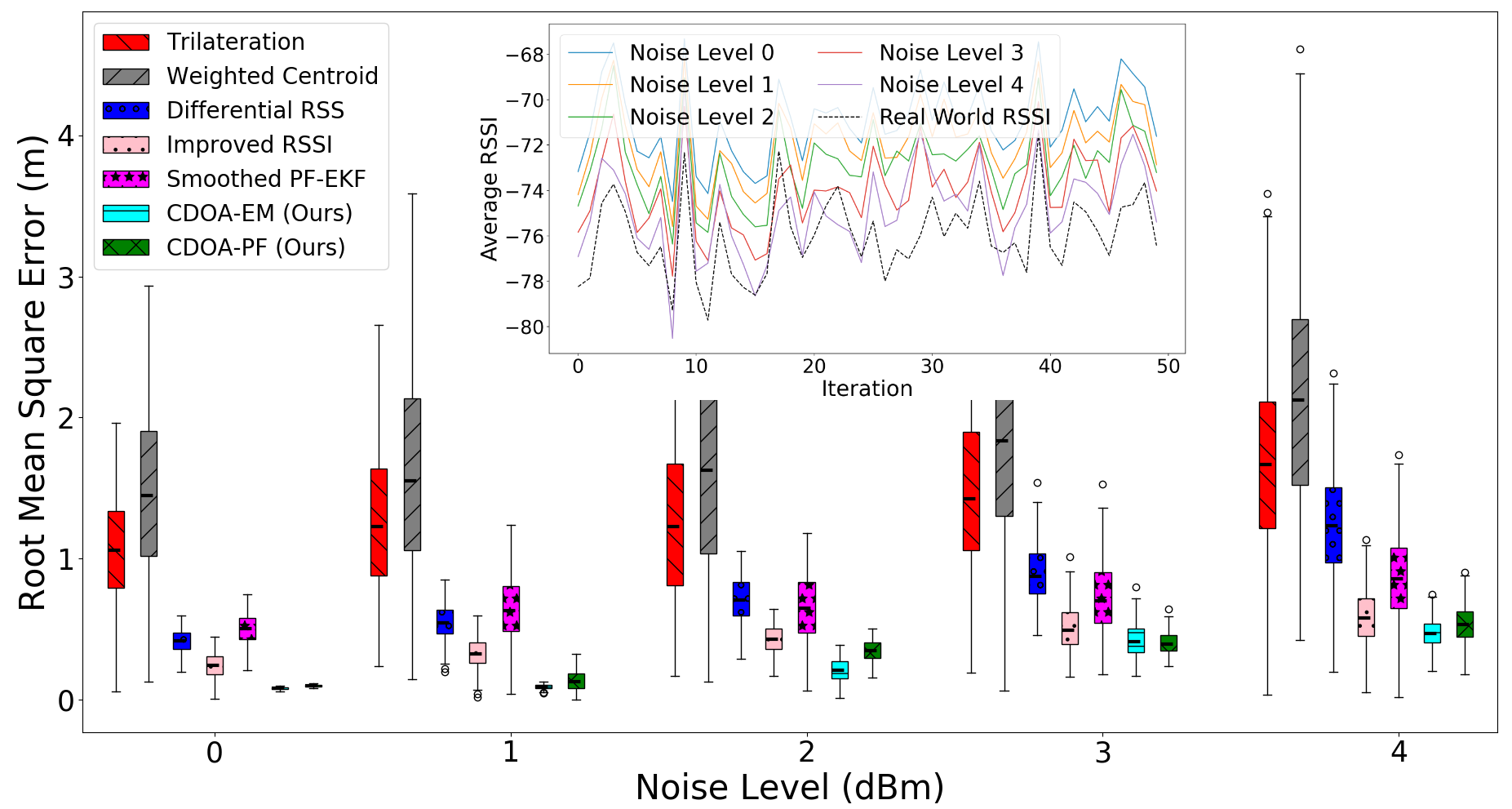}
\end{center}
\vspace{-4mm}
 \caption{\textcolor{black}{Localization performance under different RSSI noise levels. The embedded plot represents the RSSI variation at different noise levels to visualize the representativeness of the simulated RSSI with real-world data}.}
 \label{fig:simulation_noise_performance_plot}
\end{figure}

Fig.~\ref{fig:sim-results} summarizes the results of the simulation experiments. 
The proposed CDOA approaches outperformed all SOTA methods and have been shown to localize a mobile device using an existing WSN or AP infrastructure with up to 8 cm accuracy achieved in our simulation environment of 6x6 m, even in high signal noise of 4dBm. The CDOA-EM method provided the best accuracy among all methods, while the I-RSSI approach had high localization accuracy among the SOTA algorithms. Furthermore, in comparison to efficiency, CDOA-PF is 40\% more efficient than CDOA-EM in simulations, as expected. 
\textcolor{black}{SBL-DOA demonstrate comparable localization accuracy than proposed approach but is 10\% less efficient than CDOA-PF.}
It can be seen that the weighted centroid has the least accuracy of 84\%, and the proposed approach has the highest accuracy of 92\% compared to the ground truth location.
Also, as expected, the insider and diagonal trajectories provide better accuracies than the boundary cases for all the compared methods. The RSSI and DOA-based location estimates can be ambiguous on the boundary regions, resulting in higher localization errors. 

The PF-EKF approach used the same number of particle filters and applied an EKF to predict and update the robot's state while using raw RSSI values (which are inconsistent and unreliable \cite{dong2012reliability}) as an observation, resulting in 30\% reduced position estimation accuracy than our approach.
In addition, the EKF update and prediction step in PF-EKF added overhead in the computation and made it complex, while the proposed CDOA-PF has no such computationally expensive operation; hence time per iteration of CDOA-PF is 60\% less than that of PF-EKF. Overall, the proposed CDOA-PF achieved a balance of high localization accuracy and efficiency over all other SOTA algorithms.

\subsubsection{Ablation Analysis}
In the proposed approach, certain factors, such as the number of particles and noise level in RSSI measurements, can impact localization accuracy. 

We analyze the localization accuracy (RMSE) for different simulated noise levels in the measured RSSI values. 
It can be seen in Fig.~\ref{fig:simulation_noise_performance_plot} that the proposed approaches have lower RMSE (high accuracy) among all techniques, even under high noise levels. The accuracy improvement is more pronounced when the noise level is increased. The trilateration approach performed better than the Weighted Centroid method. However, both have 3x lower accuracy than the proposed EM and PF-based methods for most experiments. CDOA-EM performed slightly better than CDOA-PF in terms of accuracy and robustness. However, the CDOA-EM approach is computationally complex as it calculates the Gaussian probability in all the grid areas with a full resolution instead of sparsely and randomly distributed particles in the CDOA-PF.

\subsection{Real-world Public Datasets}
In general, our method outperformed other methods in both datasets, as can be seen in Table~\ref{tab:results}. Further analysis of individual technologies and channels is provided below.

\begin{figure}[t]
    \centering
    \includegraphics[width=0.98\linewidth]{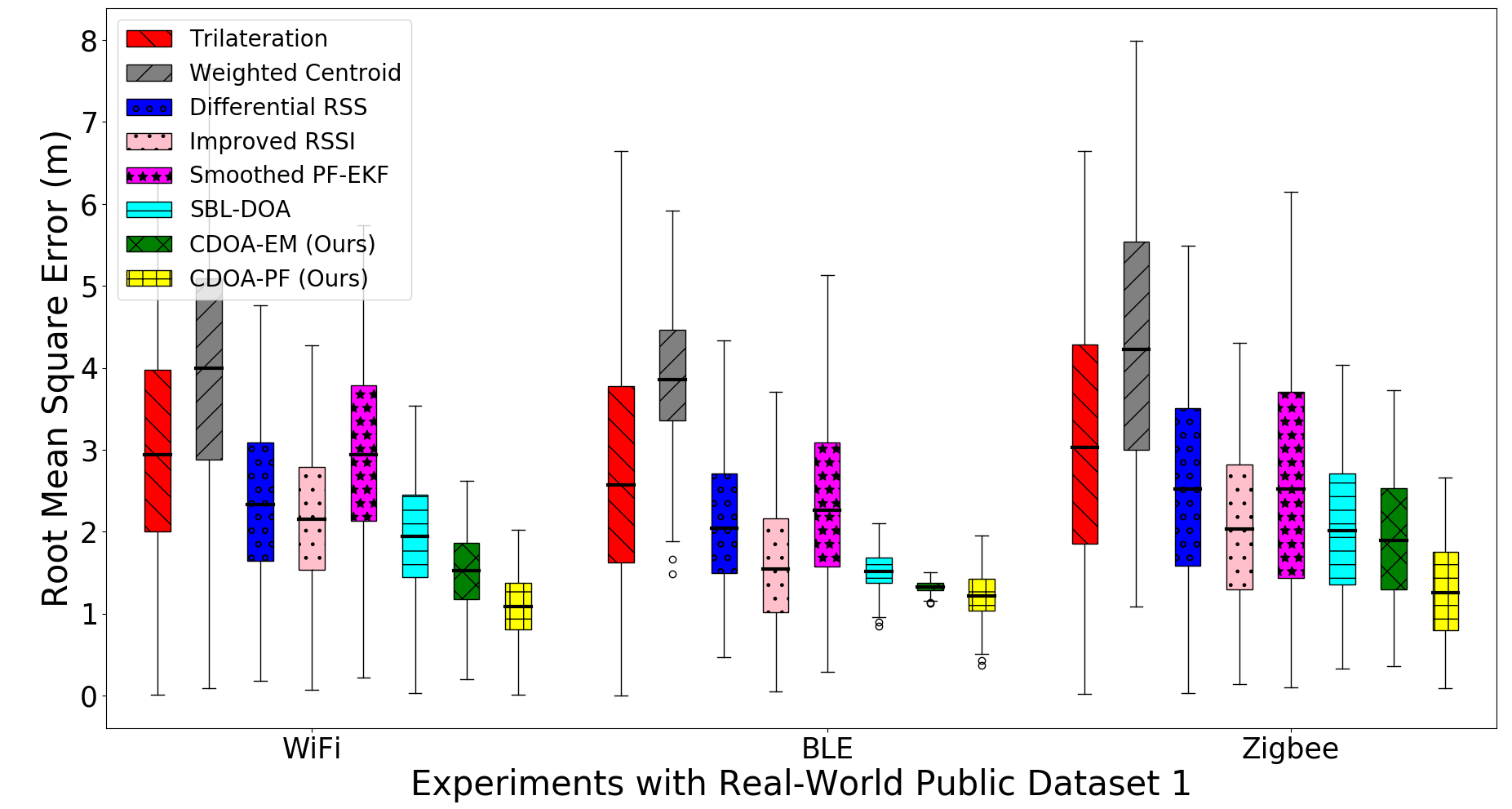}
    \vspace{-3mm}
    \caption{Comparative localization performance on the real-world dataset 1.}
     \label{fig:dataset1_result}
\end{figure}

\begin{figure}[t]
    \centering
   \includegraphics[width=0.98\linewidth]{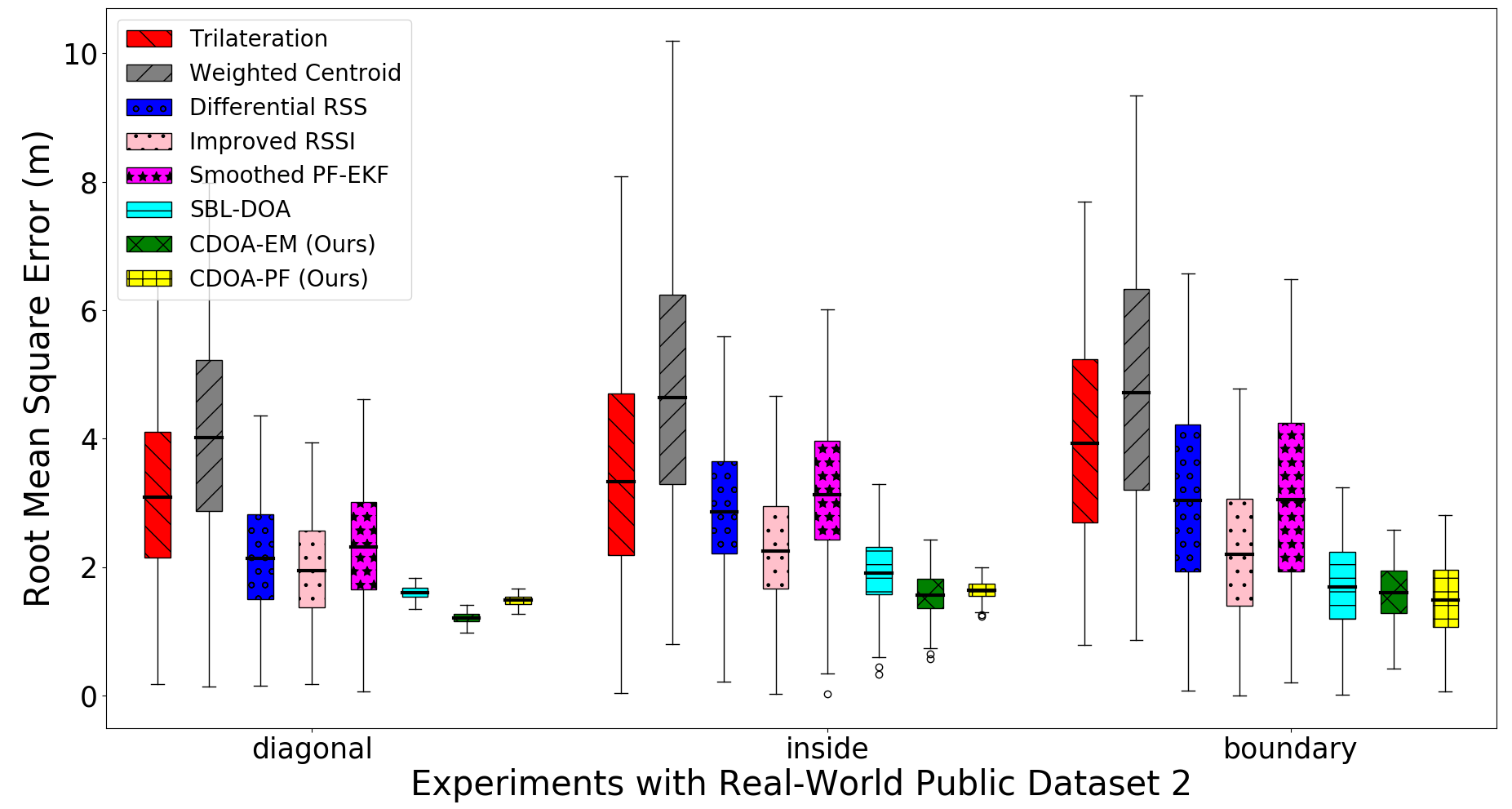}
      \vspace{-3mm}
    \caption{Comparative localization performance on the real-world dataset 2.}
     \label{fig:dataset2_result}
\end{figure}

\textbf{Dataset 1: Wireless technology comparison}
Dataset 1 captures RSSI observations from three technologies (Wi-Fi, BLE, and Zigbee). The findings from these experiments (Fig.~\ref{fig:dataset1_result}) delineate the suitability of using Wi-Fi as a communication channel and the proposed approach for indoor localization, as it has the least RMSE than other technologies, with 1.07m RMSE in a 4m x 4m of bounded region. As expected, there is no significant difference in computational complexity among different technologies because the methods take the same time to run the algorithm, irrespective of where the signals are coming from. However, CDOA-EM requires high time per iteration for position prediction than any other method. 
The proposed CDOA approach also has 0.4 m better accuracy than the KNN as presented in \cite{sadowski2020memoryless} scenario 1 of Dataset 1. \textcolor{black}{SBL-DOA used a similar DOA estimation technique and achieved relatively higher localization accuracy than other SOTA approaches but slightly less than the proposed CDOA-EM/PF.}
The CDOA-PF provided higher accuracy than the CDOA-EM, contrary to the expectation that EM provides better PF accuracy. We believe this is because the softmax-based convergence of a few particles in PF after several iterations resulted in the closest position estimate rather than a more significant number of grids converging in the EM. This is especially the case for certain positions where the bounding centroid of grids is very close to the actual location, which yields lower RMSE in CDOA-PF for longer trajectories, as found in this dataset.

\textbf{Dataset 2: Regional and Channel-wise comparison}
Dataset 2 has RSSI captured in three regions (inside, diagonal, and boundary) in a bounded size 10 x 10 m.  All the localization methods work well for the positions inside the AP/WSN perimeter, compared to diagonal or boundary points, similar to the results found in the simulations. However, the robot node would generally be inside the infrastructure boundaries, exploiting the full advantages of the proposed CDOA-based localization methods. Accordingly, the CDOA methods provided the best accuracy compared to other methods.

Dataset 2 also has RSSI observations for 40 different channels under three regions in the bounded dimension 10 x 10 m.
We show the results for the combined RMSE for channels 0-39 in Fig.~\ref{fig:dataset2_result}. It can be seen from Table.~\ref{tab:results} that the proposed approach consistently provided the best performance in most of the scenarios with reasonable computational efficiency. We also analyzed the individual channel performance and found that the comparative results did not differ significantly compared to the averaged channel estimates. However, some channels were found to have higher noise and, therefore, poorer performance on all the compared methods. \textcolor{black}{SBL-DOA approach over different channels performed more or less the same as the proposed approach because the SB-DOA relied on the Bayesian learning to mitigate the effect of channel variations, but at the cost of high computational cost, which makes it less practical in the robotics domain.}

\begin{figure}[t]
    \centering
   \includegraphics[width=0.98\linewidth]{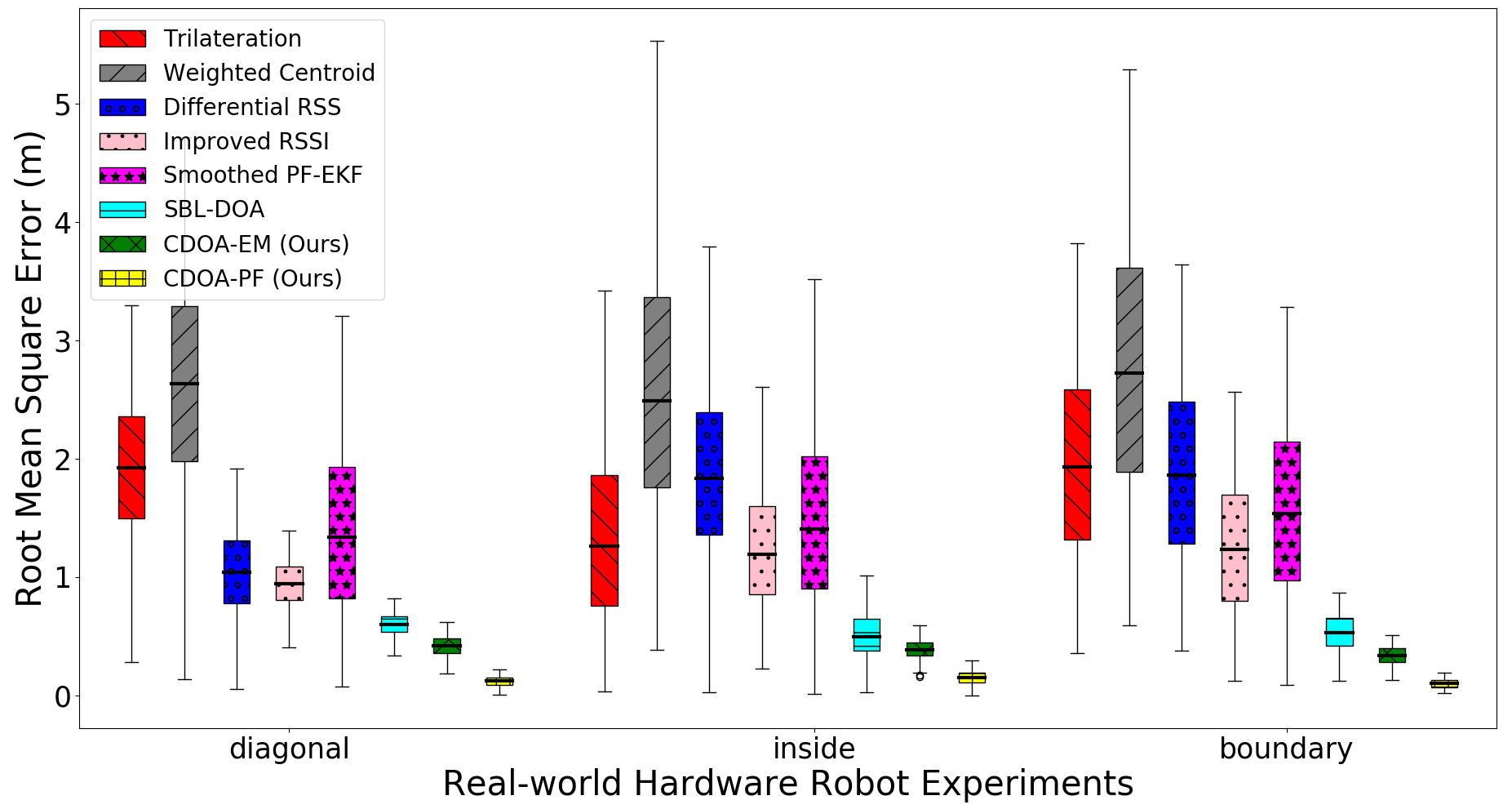}
   \vspace{-3mm}
    \caption{Comparative localization performance on the real robot experiments.}
     \label{fig:hardware_result}
\end{figure}

\subsection{Real Robot Hardware Experiments}
\label{sec:hardware-results}
Fig.~\ref{fig:hardware_result} provides the averaged results of the real robot experiments in different regions.
The results delineate high average localization accuracy of 95\%, with an absolute RMSE of 0.12 m (in a bounded region of $4 m^2$) for CDOA methods in our hardware experiments.
A sample of the hardware experiments and demonstration can be viewed at \url{https://www.youtube.com/watch?v=jVg2hzouO9E}. \textcolor{black}{As compared to other SOTA; SBL-DOA performed better and has shown 90\% localization accuracy but 5\% less accurate than the proposed CDOA method.}

Table~\ref{tab:results} also provides comparative numerical values of localization error and time per iteration for other SOTA approaches, which validate the significant performance improvement achieved by the proposed CDOA methods in all trajectories. CDOA-EM has relatively better localization than different approaches but with a high computational cost. The CDOA-PF has more than 93\% localization compared to other model-based approaches, confirming the simulation experiments. Interestingly, the CDOA-PF is 35\% more accurate than CDOA-EM, similar to the observation in dataset 1.

Furthermore, the hardware experiments also validated the practicality of the proposed CDOA approach by providing evidence of high localization accuracy, similar to simulation and dataset results.
Finally, we have made a hardware experiments dataset and the software source codes of the implementations available for the research community to reproduce the results and build on the work to improve DOA-based localization approaches further. 
\textcolor{black}{The proposed CDOA-PF and CDOA-EM approaches, given their impressive localization accuracy in both simulated and real-world environments, should operate effectively even in non-line-of-sight (NLOS) scenarios. These methods, especially the CDOA-PF, combine robustness to variations in the signal environment with computational efficiency, making them likely to handle NLOS conditions where signals may be blocked or distorted. Moreover, as these methods have shown superior performance in high signal noise scenarios, they should be robust to the increased noise and multipath effects commonly found in NLOS situations.}

\textit{Limitations:} Similar to any wireless node collaboration-based approach \cite{wang2019assistant}, the proposed approach only works for more than three fixed nodes placed at geometrically-aligned positions of a regular polygon bounded region to obtain accurate CDOA. While it is robust for most scenarios, it depends on the quality of the RSSI and the obstacles or non-line of sight conditions, which need to be studied further.

\section{Conclusion}
We proposed a novel collaborative direction of arrival (CDOA) based instantaneous localization method for robotic wireless nodes that can collaborate with other static wireless/IoT nodes in a GPS-denied environment. The proposed CDOA method efficiently incorporates wireless signals from multiple sources for position estimation. The received RSSI is used to determine the Gaussian probability of wireless signal CDOA through the geometric properties of the wireless sensor nodes. Two Bayesian approaches (Expectation Maximization and Particle Filter) use the CDOA probability density function and best fit (high weight and less probability noise) for the robot node state (position) estimation. 

Extensive simulation experiments, as well as real-world datasets and demonstrations, have revealed promising results on the method's accuracy and efficiency. 
Our method revealed at least a 50\% improvement in localization accuracy than other non-sampling-based methods such as trilateration, weighted centroid, etc., using RSSI-only data, and 20\% better than sampling-based techniques. Furthermore, despite CDOA being a sample-based technique, due to the less computational cost of CDOA and less frequent resampling, it is more efficient than other sampling and non-sampling-based methods. The experiments proved the practicality of the approach for cooperative robot localization to achieve high accuracy with low computational costs.

\bibliography{ehsan_bib,pfdoa_bib}
\bibliographystyle{IEEEtran}

\clearpage

\appendices

\section{Compared State-of-the-Art Methods}
\label{sec:appendix-methods}
Below are the specific details of the methods implemented from the state-of-the-art fingerprint-less localization techniques so as to compare with the proposed CDOA methods.

\subsection{Trilateration \cite{fundamentalTrilateration1996} }
Trilateration is a historical model-based technique \cite{fundamentalTrilateration1996} that uses distances to determine the receiver's location numerically. To calculate with trilateration, we need three transmitting devices to obtain a 2-D position and four to find a 3-D position. The distances between the transmitter and the receivers and the right number of transmitting devices are necessary. A frequent method for calculating the distance between devices is to use the RSSI of a signal. 
For 2-D space, with three anchor nodes $N_1,N_2,N_3$ and positions in space be $(a_1,b_1),(a_2,b_2),(a_3,b_3)$ respectively. We can find the unknown position $(x,y)$ of the receiver as:
 \begin{align*}&\begin{cases} \displaystyle (a_{1}-x)^{2}+(b_{1}-y)^{2}=d_{1}^{2} \\ \displaystyle (a_{2}-x)^{2}+(b_{2}-y)^{2}=d_{2}^{2} \\ \displaystyle (a_{3}-x)^{2}+(b_{3}-y)^{2}=d_{3}^{2} \end{cases} 
\end{align*}
To minimize the posting error, we need to minimize the following objective function using a non-linear least squares technique:
\begin{equation*} f(x,y)=\sum _{i=1}^{3}\left [{\sqrt {(x-a_{i})^{2}+(y-b_{i})^{2}}-d_{i}}\right]^{2}\end{equation*}

\subsection{Weighted Centroid \cite{weightedCentroid2014survey} }
The basic idea of a weighted centroid localization algorithm (e.g., \cite{weightedCentroid2014survey}) based on RSSI is that unknown nodes gather RSSI information from the beacon nodes around them.
Assuming there are n anchor nodes in the wireless sensor network, with coordinates $(x_1, y_1)$, $(x_2, y_2),. . .,(x_n, y_n)$, respectively, the location of the unknown node can be obtained by using the improved centroid algorithm estimating the coordinates of $n$ nodes as:
\begin{align*} &\begin{cases} x=\frac{{w_{1}}^{\ast}{x_{1}}^{+}{w_{2}}^{\ast}x_{2}{w_{3}}^{\ast}{x_{3} }^{+\ldots+}{w_{n}}^{\ast}x_{n}} {{w_{1}}^{+}{w_{1}}^{+}{w_{1}}^{+\ldots+}w_{1}}\\ y=\frac{{w_{1}}^{\ast}{y_{1}}^{+}{w_{2}}^{\ast}y_{2}{w_{3}}^{\ast}{y_{3} }^{+\ldots+}{w_{n}}^{\ast}y_{n}} {{w_{1}}^{+}{w_{1}}^{+}{w_{1}}^{+\ldots+}w_{1}}\\ \end{cases}\\ &w_{i}=\frac{RSSI_{i}}{RSSI_{1}+RSSI_{2}+RSSI_{3}+\ldots+RSSI_{n}}\\ &i\in(1,2,3,\ldots,n)\end{align*}

\subsection{Differential RSS \cite{podevijn2018comparison} }
The Differential RSS method in \cite{podevijn2018comparison} works without knowing to transmit power beforehand. There are two phases in this technique; offline and online phases. During the offline phase, received RSS values are generated using the representative (measured) RSS model for each grid point. During the online phase, the measured DRSS values for each grid point are compared to the theoretical ones. The estimated location (X, Y) is determined as the grid point with the theoretical RSS values closest (the least squares) to the ones measured:
\begin{equation*} (X,Y)=\min_{x,y}\sum_{i=1}^{N}(DRSS_{(x,y),i,T}-DRSS_{i,M})^{2} \end{equation*}
Here $DRSS_{(x,y), I, T}$ denotes the actual differential RSS value at position (x,y) from or at anchor $i$, $i=1 ... N$ with N is the number of anchors.$RSS_{i, M}$ is the measured RSS value from or at anchor $i$ and is therefore required. $DRSS_0$ is obtained from a measurement at a reference point using the below equation.
\begin{equation*} DRSS_{i}=RSS_{i}-RSS_{1} \end{equation*}
Here, $RSS_1$ denotes the most significant received signal strength. In the algorithm, $n$ is considered constant and known.

\subsection{Improved RSSI \cite{Xue2017improvedrssi} }
The mean or probability of the RSSI signal uses to determine the RSSI signal characteristics. However, because of multipath and non-line-of-sight propagation in a complex and dynamic interior environment, RSSI characteristics can vary significantly in time and space. For example, if multipath interference is higher than the signal, the probability distribution is right-skewed. As a result, the RSSI signal's mean does not adequately represent RSSI's dynamic nature. Therefore, the improved RSSI algorithm \cite{Xue2017improvedrssi} works to bag RSSIs for a particular time interval to maintain the temporal correlation between observations, then extract top $k$ values from decidedly sorted observations (value of $k$ determined to be 13 for indoor localization). Later on, average out the extracted observation to find the finest RSSI will then be used to calculate the differential distance from the previous observation using the differential of standard signal intensity attenuation model as:
\begin{equation*}
    \Delta d = d_o 10^{\frac{A-R_{t}}{10\eta}} - d_o 10^{\frac{A-R_{t-1}}{10\eta}},
\end{equation*}
where $d_o$ set to one meter, $A$ is received RSSI at $d_o$, $R_{t}$ and $R_{t-1}$ are the processed RSSI at time $t$ and $t-1$ respectively.
For localization, an initial position is (0, 0), and subsequent positions are estimated using the intersection of $\Delta d$ measured from each access point using standard rules of trilateration.

\subsection{Particle Filter Extended Kalman Filter (PF-EKF) \cite{zafari2018baysianfiltering} }
For indoor localization, Bayesian filtering is an appealing solution. Which, on the other hand, requires the system(depicts how the state changes over time.) and measurement (relates the noisy measurements (RSSI for PF and the user position for EKF) with the state/position) models. Using a recursive filter, the PF-EKF algorithm in \cite{zafari2018baysianfiltering} calculates the posterior Probability Density Function (PDF). The prediction and update stages of recursive filters are where the state predicts and then updates once the measurements are available. Then, using the Bayes theorem, the updated state has gathered measurements to adjust the prediction PDF. 
The algorithm first uses a particle filter to find the estimated PDF in weighted random samples as:
\begin{equation*}
 p(y_{i}\vert x_{1:i}) \approx\sum\limits_{k=1}^{N_{s}}w_{i}^{k}\delta(y_{i}-y_{i}^{k}) 
\end{equation*}
Furthermore, we used this PDF in an extended Kalman filter to get the prediction of position and update the measurement model for future estimation as:\\
\texttt{Predict:}
\begin{gather*}
\bar{Y}_{i-1}= FY_{i-1} \\
\bar{P}_{i-1}= FP_{i-1}F^{T}+Q
\end{gather*}
\texttt{Update:}
\begin{gather*}
K_{i}=\bar{P}_{i-1}H^{T}(H\bar{P}_{i-1}H^{T}+R)^{-1}\\
\bar{Y}_{i}=\bar{Y}_{i-1}+ K_{i}(X_{i}-H\bar{Y}_{i-1})\\
P_{i}=\bar{P}_{i-1}(1-KH)
\end{gather*}

{\color{black}
\subsection{SBL-DOA \cite{wang2019assistant}.}
The proposed method in \cite{wang2019assistant} focuses on assisted vehicle localization based on three collaborative base stations and SBL-based robust DOA estimation. To accurately estimate the direction of arrival (DOA) for vehicle localization, the authors create the sparse Bayesian learning (SBL) technique. The main concept is to use three base stations for cooperative localization and to take advantage of the previously known sparsity in the angular space. 
The SBL-based robust DOA estimation problem is formulated as:
\begin{equation*}
\begin{aligned}
& \underset{\boldsymbol{\theta}, \boldsymbol{\alpha}}{\text{minimize}}
& & \frac{1}{2} |\mathbf{y} - \mathbf{A}(\boldsymbol{\theta})\mathbf{x}|{2}^{2} + \frac{\gamma}{2} |\mathbf{x}|{1} \\
& \text{subject to}
& & \alpha_{i} = \frac{1}{\sigma_{i}^{2}}, \quad i = 1, \ldots, N,
\end{aligned}
\end{equation*}
where $\mathbf{y}$ represents the received signal, $\mathbf{A}(\boldsymbol{\theta})$ denotes the steering matrix, $\mathbf{x}$ is the sparse vector of interest, $\boldsymbol{\alpha}$ is the hyperparameter vector, and $\gamma$ is the regularization parameter.

After obtaining the DOA estimates, the proposed method triangulates the robot's location using the base stations. For a system with three base stations, the vehicle's position $(x, y)$ is calculated using the following set of equations:
\begin{equation*}
\begin{aligned}
& d_{1} = \sqrt{(x - x_{1})^{2} + (y - y_{1})^{2}} \\
& d_{2} = \sqrt{(x - x_{2})^{2} + (y - y_{2})^{2}} \\
& d_{3} = \sqrt{(x - x_{3})^{2} + (y - y_{3})^{2}},
\end{aligned}
\end{equation*}
where $(x_{i}, y_{i})$ denotes the position of the $i$-th base station and $d_{i}$ represents the distance between the vehicle and the $i$-th base station.

By applying the SBL-based robust DOA estimation and triangulation, the proposed approach in \cite{wang2019assistant} achieves accurate and reliable vehicle localization in the presence of multipath and shadowing effects commonly encountered in urban environments.

}

\begin{figure}[ht]
\centering
\includegraphics[width=0.75\linewidth]{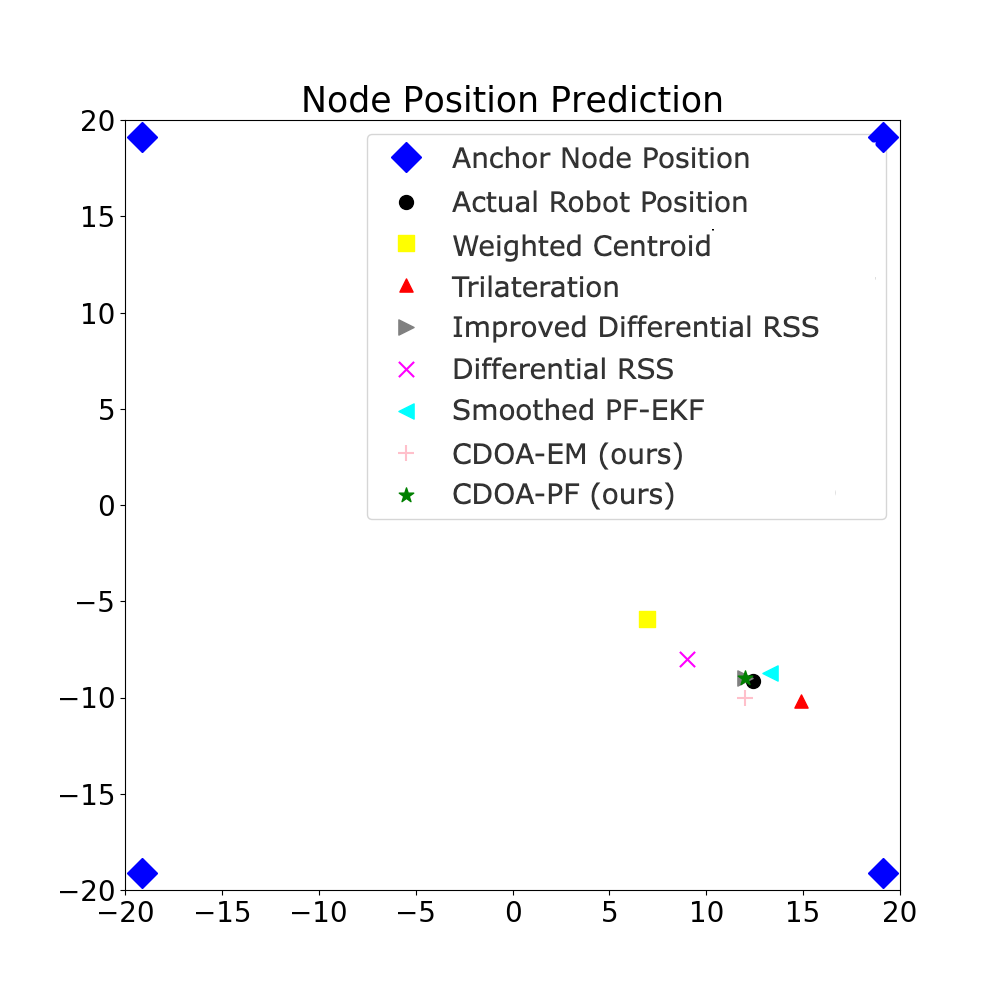}
\vspace{-4mm}
\caption{A sample output comparison of different localization algorithms in a 40 x 40 m bounded region.}
 \label{fig:scatter}
\end{figure}

\subsection{Illustration of a localization example}
We plotted the position predictions in a large simulated area of 40 m x 40 m bounded region to further understand how our method localizes compared to other methods. Fig.~\ref{fig:scatter} visualizes the localization accuracy for all methods for a specific sample location. We can observe that the proposed methods predicted robot node location very close to the actual location, even in a large area.

\subsection{Ablation study on the number of particles in CDOA-PF}
We present an ablation analysis of the number of particles in the proposed CDOA-PF method since it is a sampling-based method depending heavily on the number of particles. 
Fig.~\ref{fig:sim-particles} presents the results of the RMSE variations under different numbers of particles in the proposed CDOA-PF approach. The accuracy has improved by increasing the number of particles, but not so much beyond a certain optimum value. Also, more particles increased the computational complexity. Implementing the motion model over PF significantly improved the accuracy at a small expense in the computation time. The results demonstrate the potential of integrating the mobile robot's motion model (or IMU sensor), when possible, to improve the CDOA-based localization solution.

\begin{figure}[t]
\centering
 \includegraphics[width=0.75\linewidth]{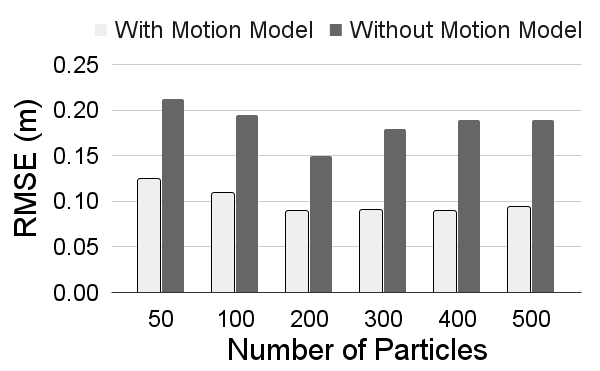}
  \vspace{-2mm}
\caption{Impact on the localization accuracy of CDOA-PF by the number of particles and inclusion of robot motion model (odometry).}
 \label{fig:sim-particles}
\end{figure}

\end{document}